\documentclass{article}

\usepackage{iclr2022_conference,times}
\setcitestyle{square,numbers}
\setcitestyle{citesep={,}}

\usepackage{amsmath,amsfonts,bm}

\def\Figref#1{Figure~\ref{#1}}

\def\Secref#1{Section~\ref{#1}}

\def\eqref#1{equation~\ref{#1}}
\def\Eqref#1{Equation~\ref{#1}}

\def\1{\bm{1}}

\def\rvu{{\mathbf{i}}}

\def\rvr{{\mathbf{r}}}

\def\rvu{{\mathbf{u}}}
\def\rvv{{\mathbf{v}}}

\def\rvx{{\mathbf{x}}}
\def\rvy{{\mathbf{y}}}

\DeclareMathAlphabet{\mathsfit}{\encodingdefault}{\sfdefault}{m}{sl}
\SetMathAlphabet{\mathsfit}{bold}{\encodingdefault}{\sfdefault}{bx}{n}

\def\gB{{\mathcal{B}}}
\def\gC{{\mathcal{C}}}
\def\gD{{\mathcal{D}}}

\def\gL{{\mathcal{L}}}

\def\gN{{\mathcal{N}}}

\def\gU{{\mathcal{U}}}

\def\sR{{\mathbb{R}}}

\newcommand{\E}{\mathbb{E}}

\usepackage{graphicx}
\usepackage{caption} 
\usepackage{zref}
\usepackage{xcolor}
\usepackage{booktabs} 
\usepackage{amsmath,multirow}
\usepackage{amssymb}
\usepackage{enumerate}
\usepackage{subfigure}
\usepackage{tabulary}
\usepackage{url}            
\usepackage{amsfonts}       
\usepackage{nicefrac}       
\usepackage{microtype}      
\usepackage{booktabs}
\usepackage{amsfonts}
\usepackage{amstext}
\usepackage{amsthm}
\usepackage{bm}
\usepackage{wrapfig}
\usepackage{bbm, comment}
\usepackage{xcolor}
\usepackage{float}
\usepackage{thm-restate}
\usepackage{MnSymbol,wasysym,marvosym,ifsym}
\usepackage{algorithm}
\usepackage{algorithmic}
\usepackage[utf8]{inputenc}
\usepackage[colorlinks=true,citecolor=brown,urlcolor=gray]{hyperref}

\newcommand{\mat}[1]{\mathbf{#1}}

\newcommand{\zmin}{z_{\textrm{min}}}
\newcommand{\zmax}{z_{\textrm{max}}}
\newenvironment{proofs}{%
  \proof}{\endproof}

\usepackage{amsmath,amsfonts,bm,amssymb,mathtools,amsthm}
\usepackage{color,xcolor}
\usepackage{amsmath}
\usepackage{xspace} 
\def\onedot{$\mathsurround0pt\ldotp$}
\def\eg{\emph{e.g}\onedot, }

\def\ie{\emph{i.e}\onedot, }
\newtheorem{theorem}{Theorem}

\newtheorem{proposition}{Proposition}

\newtheorem{lemma}{Lemma}

\def\Figref#1{Fig.~\ref{#1}}

\def\Secref#1{Section~\ref{#1}}

\def\eqref#1{equation~\ref{#1}}
\def\Eqref#1{Eq.~(\ref{#1})}

\def\1{\bm{1}}

\def\rvu{{\mathbf{i}}}

\def\rvr{{\mathbf{r}}}

\def\rvu{{\mathbf{u}}}
\def\rvv{{\mathbf{v}}}

\def\rvx{{\mathbf{x}}}
\def\rvy{{\mathbf{y}}}

\DeclareMathAlphabet{\mathsfit}{\encodingdefault}{\sfdefault}{m}{sl}
\SetMathAlphabet{\mathsfit}{bold}{\encodingdefault}{\sfdefault}{bx}{n}

\def\gB{{\mathcal{B}}}
\def\gC{{\mathcal{C}}}
\def\gD{{\mathcal{D}}}

\def\gL{{\mathcal{L}}}

\def\gN{{\mathcal{N}}}

\def\gU{{\mathcal{U}}}

\def\sR{{\mathbb{R}}}

\newcommand\numberthis{\addtocounter{equation}{1}\tag{\theequation}}
\iclrfinalcopy

\usepackage{booktabs,arydshln}
\usepackage{amssymb}
\usepackage{pifont}
\newcommand{\cmark}{\ding{51}}%
\newcommand{\xmark}{\ding{55}}%

\makeatletter
\def\adl@drawiv#1#2#3{%
        \hskip.5\tabcolsep
        \xleaders#3{#2.5\@tempdimb #1{1}#2.5\@tempdimb}%
                #2\z@ plus1fil minus1fil\relax
        \hskip.5\tabcolsep}
\newcommand{\cdashlinelr}[1]{%
  \noalign{\vskip\aboverulesep
          \global\let\@dashdrawstore\adl@draw
          \global\let\adl@draw\adl@drawiv}
  \cdashline{#1}
  \noalign{\global\let\adl@draw\@dashdrawstore
          \vskip\belowrulesep}}
\makeatother

\title{Poisson Flow Generative Models}
\author{Yilun Xu\thanks{Equal Contribution.} , Ziming Liu\footnotemark[1] , Max Tegmark, Tommi Jaakkola\\
Massachusetts Institute of Technology\\
\texttt{\{ylxu, zmliu, tegmark\}@mit.edu};\quad \texttt{tommi@csail.mit.edu} \\
}

\begin{document}

\maketitle

\begin{abstract}

We propose a new ``Poisson flow" generative model~(PFGM) that maps a uniform distribution on a high-dimensional hemisphere into any data distribution. We interpret the data points as electrical charges on the $z=0$ hyperplane in a space augmented with an additional dimension $z$, generating a high-dimensional electric field (the gradient of the solution to Poisson equation). We prove that if these charges flow upward along electric field lines, their initial distribution in the $z=0$ plane transforms into a distribution on the hemisphere of radius $r$ that becomes {\it uniform} in the $r \to\infty$ limit. To learn the bijective transformation, we estimate the normalized field {in the augmented space}. For sampling, we devise a backward ODE that is anchored by the physically meaningful additional dimension: the samples hit the (unaugmented) data manifold when the $z$ reaches zero. 
{Experimentally, PFGM achieves current state-of-the-art performance among the normalizing flow models on CIFAR-10, with an Inception score of $9.68$ and a FID score of $2.35$. It also performs on par with the state-of-the-art SDE approaches while offering $10\times $ to $20 \times$ acceleration on image generation tasks. Additionally, PFGM appears more tolerant of estimation errors on a weaker network architecture and robust to the step size in the Euler method.} 
The code is available at \url{https://github.com/Newbeeer/poisson_flow}.
\end{abstract}

\section{Introduction}

Deep generative models are a prominent approach for data generation, and have been used to produce high quality samples in image~\citep{Brock2019LargeSG}, text~\citep{Brown2020LanguageMA} and audio~\citep{Oord2016WaveNetAG}, as well as improve semi-supervised learning~\citep{Kingma2014SemisupervisedLW}, domain generalization~\citep{Li2021SemanticSW} and imitation learning~\citep{Ho2016GenerativeAI}. However, current deep generative models also have limitations, such as unstable training objectives (GANs~\citep{Brock2019LargeSG,Gulrajani2017ImprovedTO,Karras2020TrainingGA}) and low sample quality (VAEs~\citep{Kingma2014AutoEncodingVB}, normalizing flows~\citep{Dinh2017DensityEU}). {New techniques~\cite{Gulrajani2017ImprovedTO,Lee2021ViTGANTG} are introduced to stablize the training of CNN-based or ViT-based GAN models.} Although recent advances on diffusion~\citep{Ho2020DenoisingDP} and scored-based models~\citep{Song2021ScoreBasedGM} achieve comparable sample quality to GAN's without adversarial training, these models have a slow stochastic sampling process. \cite{Song2021ScoreBasedGM} proposes backward ODE samplers (normalizing flow) that speed up the sampling process but these methods have not yet performed on par with the SDE counterparts.


\begin{figure*}[htbp]
    \centering
    \subfigure[]{\label{fig:heart-sphere}\includegraphics[width=0.45\textwidth, trim=2cm 6cm 2cm 6cm]{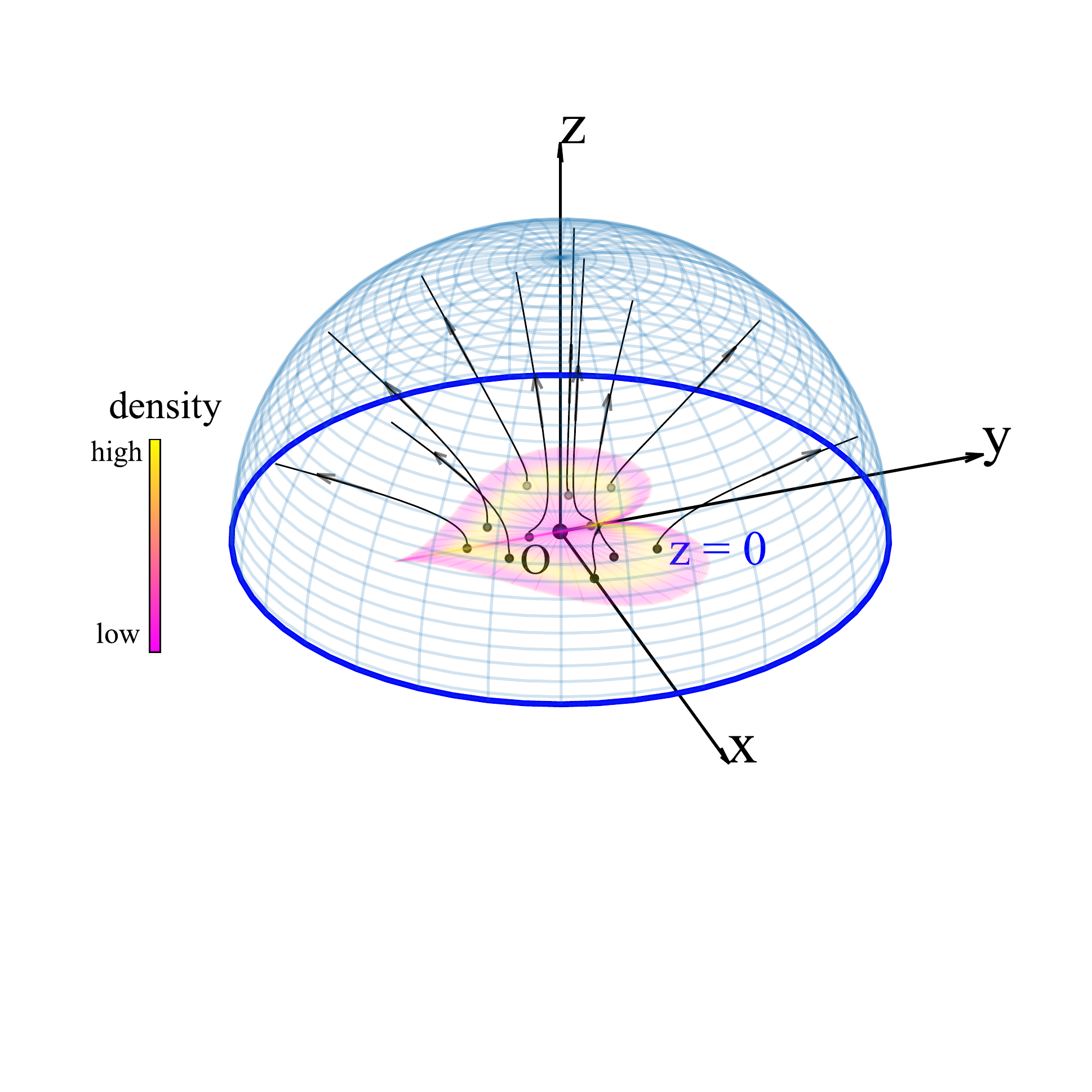}}\hspace{10pt}
    \subfigure[]{\label{fig:heart-ode}\includegraphics[width=0.50\textwidth, trim={0cm 2cm 0cm 2cm}]{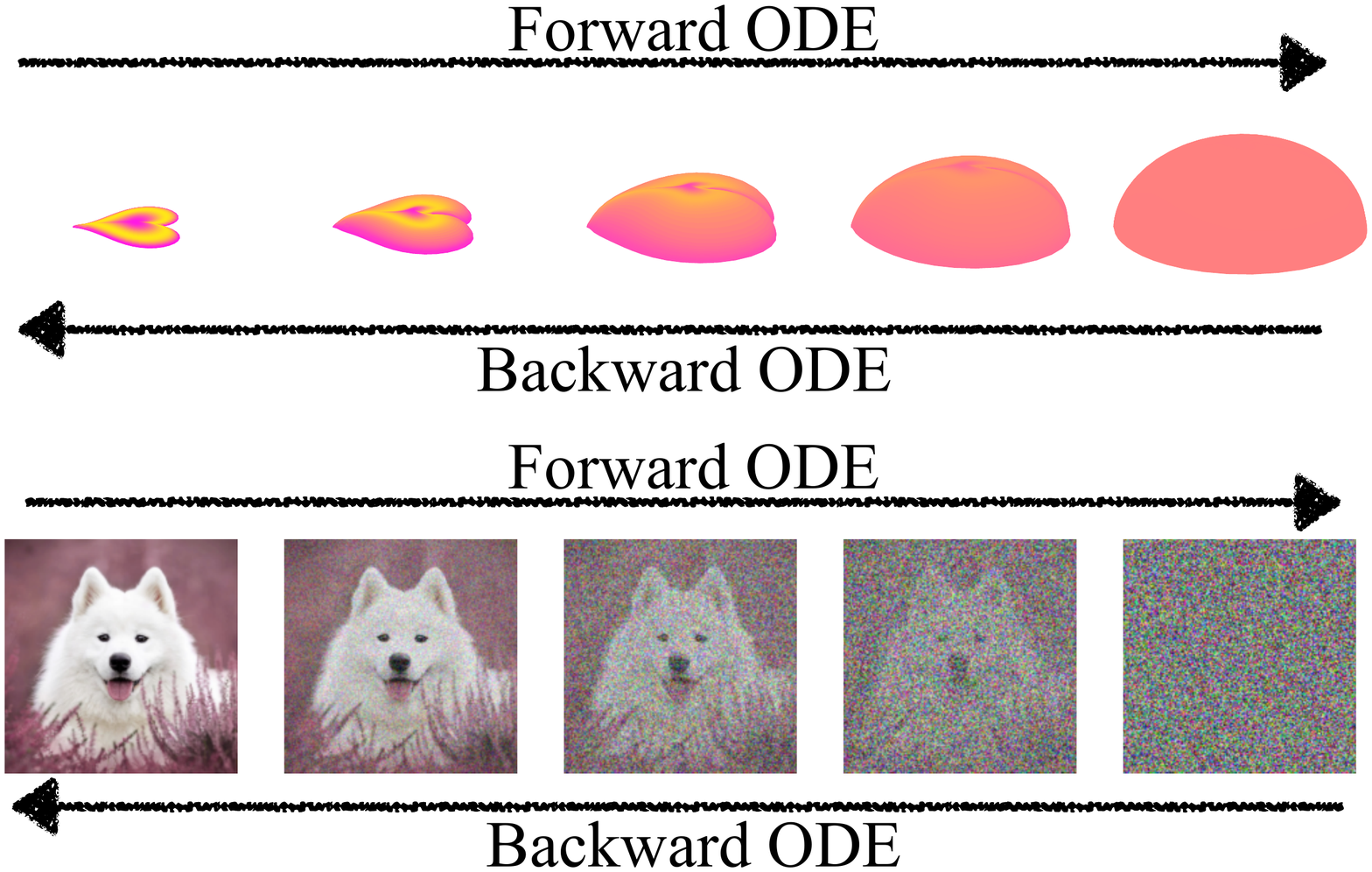}}
    \caption{\textbf{(a)} 3D Poisson field trajectories for a heart-shaped distribution \textbf{(b)} The evolvements of  a distribution~(\textbf{top}) or an (augmented) sample~(\textbf{bottom}) by the forward/backward ODEs pertained to the Poisson field.}
    \label{fig:elect_field}
\end{figure*}

We present a new ``Poisson flow" generative model (\textbf{PFGM}), exploiting a remarkable physics fact that generalizes to $N$ dimensions. As illustrated in \Figref{fig:heart-sphere}, motion in a viscous fluid transforms any planar charge distribution into a uniform angular distribution. Specifically, we interpret $N$-dimensional data points $\rvx$ (images, say) as positive electric charges in the $z=0$ plane of an $N+1$-dimensional space (see \Figref{fig:heart-sphere}) filled with a viscous liquid (say honey). A positive charge with $z>0$ will be repelled by the other charges and move in the direction of their repulsive force, eventually crossing an imaginary hemisphere of radius $r$. We show that, remarkably, if the the original charge distribution is let loose just above $z=0$, this law of motion will cause a {\it uniform} distribution for their hemisphere crossings in the $r\to\infty$ limit.

Our Poisson flow generative process reverses the forward process: we generate a uniform distribution of negative charges on the hemisphere, then track their motion back to the $z=0$ plane, where they will be distributed as the data distribution. {A Poisson flow {can be viewed as} a type of continuous normalizing flows~\cite{Chen2018NeuralOD,Grathwohl2019FFJORDFC,Song2021ScoreBasedGM} in the sense that it continuously maps between an arbitrary distribution and an easily sampled one: in the previous works an $N$-dimensional Gaussian and in PFGM a uniform distribution on an $N$-dimensional hemisphere.}
In practice, we implement the Poisson flow by solving a pair of forward/backward ordinary differential equations~(ODEs) induced by the electric field~(\Figref{fig:heart-ode}) given by the $N$-dimensional version of Coulomb's law (the gradient of the solution to the Poisson's equation with the data as sources). We will interchangeably refer to this gradient as the {\it Poisson field}, since electric fields normally refer to the special case $N=3$.


{The proposed generative model PFGM has a stable training objective and empirically outperforms previously state-of-the-art continuous flow methods~\citep{Song2021DenoisingDI, Song2021ScoreBasedGM}. As a different iterative method, PFGM offers two advantages compared to score-based methods~\citep{Song2020ImprovedTF, Song2021ScoreBasedGM}. First, the ODE process of PFGM achieves faster sampling speeds than the SDE samplers in~\citep{Song2021ScoreBasedGM}. while retaining comparable performance. Second, our backward ODE exhibits better generation performance than the reverse-time ODEs of VE/VP/sub-VP SDEs~\citep{Song2021ScoreBasedGM}, as well as greater stability on a weaker architecture NSCNv2~\cite{Song2020ImprovedTF}.} The rationale for robustness is that the time variables in these ODE baselines are strongly correlated with the sample norms during training time, resulting in a less error-tolerant inference. In contrast, the tie between the anchored variable and the sample norm in PFGM is much weaker.



{Experimentally, we show that PFGM achieves current state-of-the-art performance on CIFAR-10 dataset in the normalizing flow family, with FID/Inception scores of $2.48/9.65$~(w/ DDPM++~\cite{Song2021ScoreBasedGM}) and $2.35/9.68$~(w/ DDPM++ deep~\cite{Song2021ScoreBasedGM}). It performs competitively with current state-of-the-art SDE samplers~\cite{Song2021ScoreBasedGM} and provides $10\times$ to $20\times$ speed up across datasets. Notably, the backward ODE in PFGM is the \textit{only} ODE-based sampler that can produce decent samples on its own on NCSNv2~\cite{Song2020ImprovedTF}, while other ODE baselines fail without corrections. In addition, PFGM demonstrates the robustness to the step size in the Euler method, with a varying number of function evaluations~(NFE) ranging from $10$ to $100$.} We further showcase the utility of the invertible forward/backward ODEs of the Poisson field on likelihood evaluation and image manipulations, and its scalability to higher resolution images on LSUN bedroom $256\times 256$ dataset.

\section{Background and Related works}
\label{section:bg}

{\bf Poisson equation} Let $\mat{x}\in\mathbb{R}^N$ and $\rho(\mat{x}):\mathbb{R}^N\to \mathbb{R}$ be a \textit{source} function. {We assume that the source function has a compact support, $\rho\in \gC^0$ and $N\ge3$}. The Poisson equation is
\begin{equation}\label{eq:poisson}
    \nabla^2\varphi(\mat{x}) = -\rho(\mat{x}),
\end{equation}
where $\varphi(\mat{x}):\mathbb{R}^N\to \mathbb{R}$ is called the \textit{potential function}, and $\nabla^2\equiv \sum_{i=1}^N \frac{\partial^2}{\partial x_i^2}$ is the Laplacian operator. It is usually helpful to define the gradient field $\mat{E}(\mat{x})=-\nabla \varphi(\mat{x})$ and rewrite the Poisson equation as $\nabla\cdot \mat{E}=\rho$, known in physics as Gauss's law~\citep{griffiths2005introduction}. The Poisson equation is widely used in physics, giving rise to Newton's gravitational theory~\citep{goldstein2002classical} and the electrostatic theory~\citep{griffiths2005introduction}, when $\rho(\mat{x})$ is interpreted as mass density or electric charge density, respectively.
$\mat{E}$ is the $N$-dimensional analog of the electric field.
The Poisson equation Eq.~(\ref{eq:poisson}) {(with zero boundary condition at infinity) admits a unique simple integral solution}~\footnote{{Eq.~(\ref{eq:poisson_solution}) is valid for $N\geq 3$. When $N=2$, the Green's function is $G(\mat{x},\mat{y})=-{\rm log}(||\mat{x}-\mat{y}||)/2\pi$. We assume $N\ge 3$ since $N$ is typically large in the relevant applications.
}}:
\begin{equation}\label{eq:poisson_solution}
    \varphi(\mat{x}) = \int\\ G(\mat{x},\mat{y})\rho(\mat{y})d\mat{y},\quad  G(\mat{x},\mat{y}) = \frac{1}{(N-2)S_{N-1}(1)} \frac{1}{||\mat{x}-\mat{y}||^{N-2}},
\end{equation}
{where $S_{N-1}(1)$ is a geometric constant representing the surface area of the unit $(N-1)$-sphere~\footnote{The $N$-sphere with radius $r$ is defined as $\{\mat{x}\in\mathbb{R}^{N+1}, ||\mat{x}||=r\}$}, and $G(\mat{x},\mat{y})$ is the extension of Green's function in $N$-dimensional space~(details in Appendix~\ref{app:green})}. The negative gradient field of $\varphi(\mat{x})$, referred as \textit{Poisson field} of the source $\rho$, is
\begin{equation}\label{eq:E_green}
    \mat{E}(\mat{x}) = -\nabla \varphi(\mat{x}) = -\int\ \nabla_\mat{x}G(\mat{x},\mat{y})\rho(\mat{y})d\mat{y},\quad \nabla_\mat{x} G(\mat{x},\mat{y}) = -\frac{1}{S_{N-1}(1)} \frac{\mat{x}-\mat{y}}{||\mat{x}-\mat{y}||^{N}}.
\end{equation}
Qualitatively, the Poisson field $\mat{E}(\mat{x})$ points away from sources, or equivalently $-\mat{E}(\mat{x})$ points towards sources, as illustrated in~\Figref{fig:elect_field}. It is straightforward to check that when $\rho(\mat{x})\to\delta(\mat{x}-\mat{y})$, we get $\varphi(\mat{x})\to G(\mat{x},\mat{y})$ and $\mat{E}(\mat{x})\to -\nabla_\mat{x} G(\mat{x},\mat{y})$. This implies that $G(\mat{x},\mat{y})$ and $-\nabla_\mat{x} G(\mat{x},\mat{y})$ can be interpreted as the potential function and the gradient field generated by a unit point source, \eg a point charge, located at $\mat{y}$. When $\rho(\mat{x})$ takes general forms but has bounded support, simple asymptotics exist for $||\mat{x}||\gg||\mat{y}||$. To the lowest order, $\mat{E}(\mat{x})=\nabla_\mat{x} G(\mat{x},\mat{y})|_{\mat{y}=\mat{0}}\sim \mat{x}/||\mat{x}||^{N}$ behaves as if it were generated by a unit point source at $\mat{y}=0$. In physics, the power law decay is considered to be long-range  (compared to exponential decay)~\citep{griffiths2005introduction}.

{\bf Particle dynamics in a Poisson field}
{The Poisson field immediately defines a flow model, where the probability distribution evolves according to the gradient flow ${\partial p_t(\mat{x})}/{\partial t}=-\nabla\cdot(p_t(\mat{x})\mat{E}(\mat{x}))$. The gradient flow is a special case of the Fokker-Planck equation~\cite{Risken1984FokkerPlanckE}, where the diffusion coefficient is zero. Intuitively we can think of $p_t(\mat{x})$ as represented by a population of particles. The corresponding (non-diffusion) case of the Itô process is the forward ODE $\frac{d\mat{x}}{dt} = \mat{E}(\mat{x})$. We can interpret the trajectories of the ODE as particles moving according to the Poisson field $\mat{E}(x)$, with initial states drawn from $p_0$.}
{The physical picture of the forward ODE is a charged particle under the influence of electric fields in the overdamped limit (details in Appendix \ref{app:honey}).}

The dynamics is also \textit{rescalable} in the sense that the particle trajectory remains the same for $\frac{d\mat{x}}{dt}=\pm f(\mat{x})\mat{E}(\mat{x})$ for {$f(\mat{x})>0, f(\mat{x})\in \gC^1$}, because the time rescaling $dt\to f(\mat{x}(t))dt$ recovers $\frac{d\mat{x}}{dt}=\pm\mat{E}(\mat{x})$. Note that the dynamics is stiff due to the power law factor in the denominator in \Eqref{eq:E_green}, posing computational challenges. Luckily the rescalablility allows us to rescale $\mat{E}(\mat{x})$ properly to get new ODEs~(formally defined later in Section~\ref{sec:sampling}) that are more amenable for sampling.

{\bf Generative Modeling via ODE} Generative modeling can be done by transforming a base distribution to a data distribution via mappings defined by ODEs. The ODE-based samplers allow for adaptive sampling, exact likelihood evaluation and modeling of continuous-time dynamics~\citep{Chen2018NeuralOD, Song2021ScoreBasedGM}. Previous works broadly fall into two lines. \citep{Chen2018NeuralOD, Chen2019ResidualFF} introduce a continuous-time normalizing flow model that can be trained with maximum likelihood by the instantaneous change-of-variables formula~\cite{Chen2018NeuralOD}. For sampling, they directly integrate the learned invertiable mapping over time. Another work~\citep{Song2021ScoreBasedGM} unifies the scored-based model~\citep{Song2019GenerativeMB, Song2020ImprovedTF} and diffusion model~\cite{Ho2020DenoisingDP} into a general diffusion process, and uses the reverse-time ODE of the diffusion process for sampling. They show that the reverse-time ODE produces high quality samples with improved architecture.

\section{Poisson Flow Generative Models}

In this section, we start with the properties of the Poisson flow in the augmented space and show how to draw samples from the data distribution by following the backward ODE of the Poisson flow~(Section~\ref{sec:augment}). We then discuss how to actually learn a normalized Poisson field from data samples through simulations of the forward ODE~(Section~\ref{sec:learning}) and present an equivalent backward ODE that allows for exponentially decay on $z$~(Section~\ref{sec:sampling}).

\subsection{Augment the data with additional dimension}
\label{sec:augment}

We wish to generate samples $\mat{x}\in\mathbb{R}^N$ from a distribution $p(\mat{x})$ supported on a bounded region. We may set the source $\rho(\mat{x})=p(\mat{x}) \in \gC^0$~\footnote{{A probability distribution $p(\mat{x})$ is a special case of ``charge density" $\rho(x)$ because $p(\mat{x})$ need to be non-negative and integrates to unity. Here we focus on applications to probability distribution of data, which is the objective to be modeled in generative modeling.}} and compute the resulting gradient field $\mat{E}(\mat{x})$ from Eq.~(\ref{eq:E_green}). Since $-\mat{E}(\mat{x})$ points towards sources, the backward ODE ${d\mat{x}}/{dt}=-\mat{E}(\mat{x})$ will take samples close to the sources. One may naively hope that the backward ODE is a generative model that recovers $p(\mat{x})$. Unfortunately, the backward ODE has the problem of mode collapse. 
We illustrate this phenomenon with a 2D uniform disk. The reverse Poisson field $-\mat{E}(\mat{x})$ on the 2D $(x,y)$-plane points towards the center of the disk $O$ (\Figref{fig:augmentation} left), so all particle trajectories (blue lines) will eventually hit $O$. If we instead add an additional dimension $z$ (\Figref{fig:augmentation} right), particles can hit different points on the disk and faithfully recover the data distribution.

\begin{figure}
    \centering
    \subfigure[]{    \label{fig:augmentation}\includegraphics[width=0.49\textwidth, trim=1cm 3cm 0.5cm 3cm]{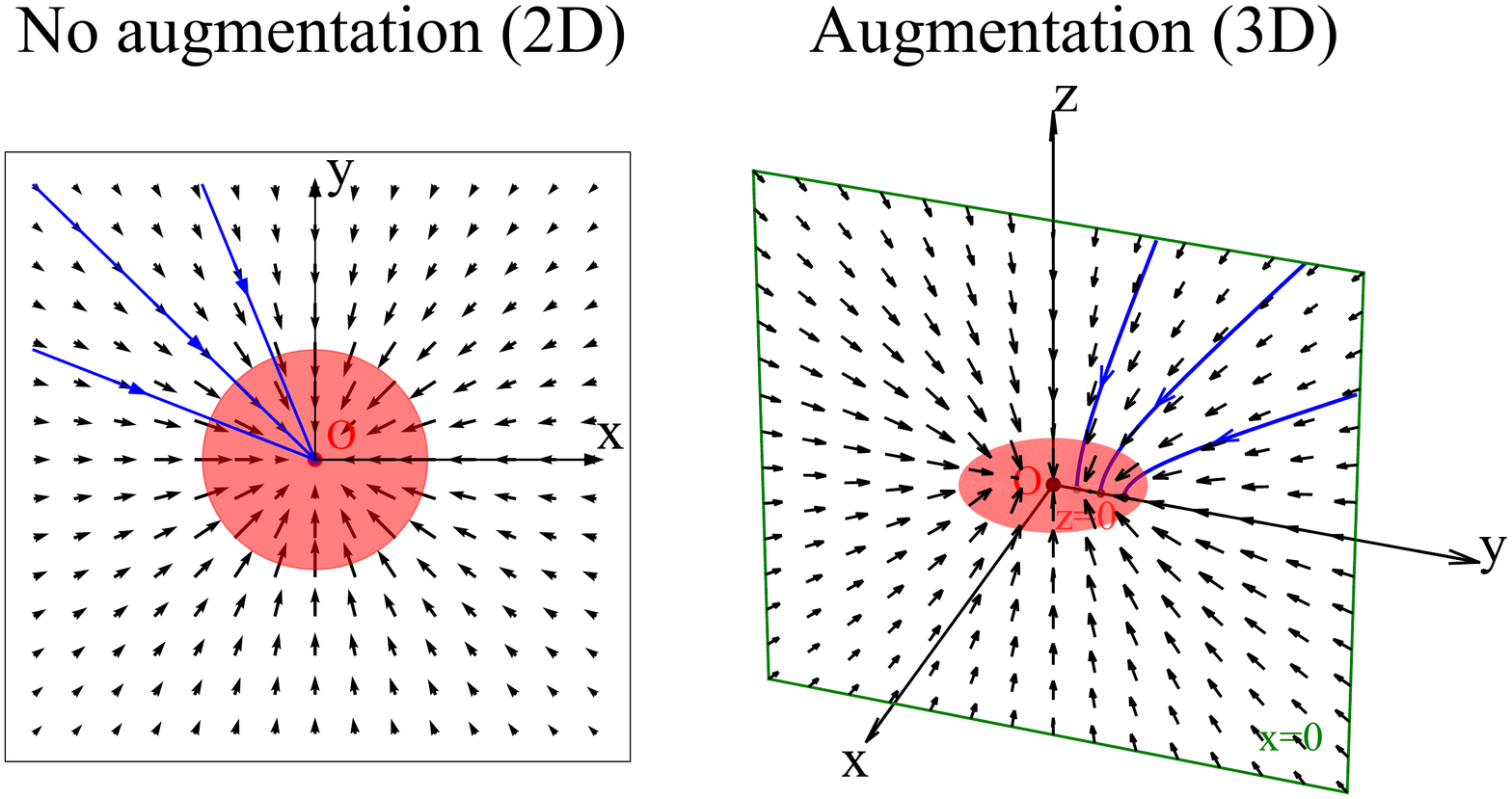}}
    \subfigure[]{\label{fig:theorem1}\includegraphics[width=0.49\textwidth, trim=0cm 3cm 0cm 3cm]{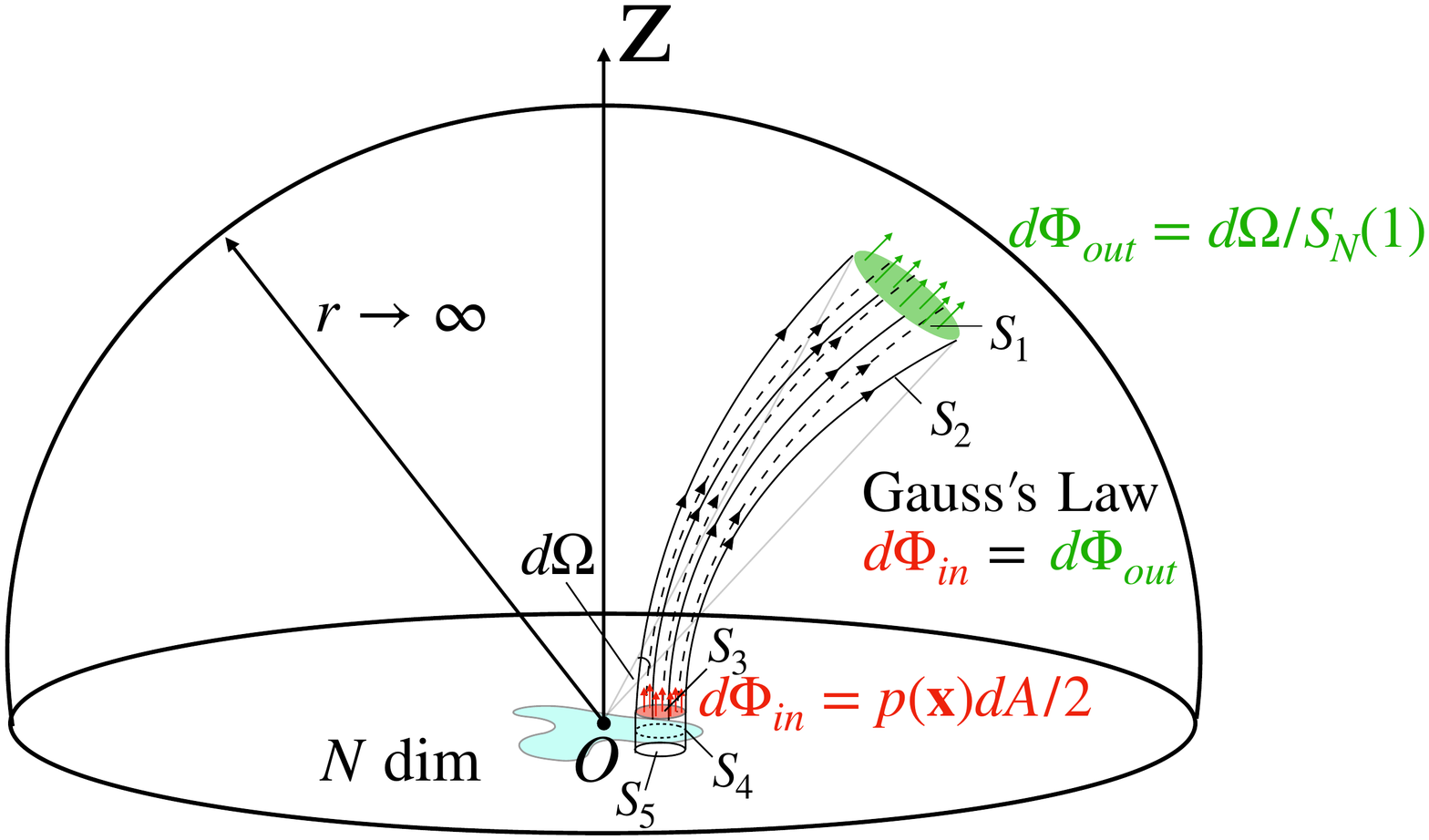}}
        \vspace{-4pt}
    \caption{\textbf{(a)} Poisson field (black arrows) and particle trajectories (blue lines) of a 2D uniform disk (red). \textbf{Left} (no augmentation, 2D): all particles collapse to the disk center. \textbf{Right} (augmentation, 3D): particles hit different points on the disk. \textbf{(b)} Proof idea of Theorem~\ref{thm:sphere}. By Gauss's Law, the outflow flux $d\Phi_{out}$ equals the inflow flux $d\Phi_{in}$. The factor of two in $p(\rvx)dA/2$ is due to the symmetry of Poisson fields in $z<0$ and $z>0$.
    }
    \vspace{-6pt}
\end{figure}

Consequently, instead of solving the Poisson equation $\nabla^2\varphi(\mat{x})=-p(\mat{x})$ in the original data space, we solve the Poisson equation in an augmented space $\tilde{\mat{x}}=(\mat{x},z) \in \sR^{N+1}$ with an additional variable $z\in\mathbb{R}$. We augment the training data $\tilde{\mat{x}}$ in the new space by setting $z=0$ such that $\tilde{\rvx}=(\rvx,0)$. As a consequence, the data distribution in the augmented space is $\tilde{p}(\tilde{\mat{x}})=p(\mat{x})\delta(z)$, where $\delta$ is the Dirac delta function. By \Eqref{eq:E_green}, the Poisson field by solving the new Poisson equation $\nabla^2\varphi(\tilde{\rvx}) = -\tilde{p}(\tilde{\rvx})$ has an analytical form:
\begin{align*}
    \forall \tilde{\mat{x}} \in \sR^{N+1}, \mat{E}(\tilde{\mat{x}}) = -\nabla \varphi(\tilde{\mat{x}}) = \frac{1}{S_{N}(1)}\int \frac{\tilde{\mat{x}}-\tilde{\rvy}}{||\tilde{\mat{x}}-\tilde{\rvy}||^{N+1}}\tilde{p}(\tilde{\rvy}) d \tilde{\rvy}\numberthis \label{eq:gradient-field}
\end{align*}
The associated forward/backward ODEs of the Poisson field are
$
     {d\tilde{\mat{x}}}/{dt} = \mat{E}(\tilde{\mat{x}}),  {d\tilde{\mat{x}}}/{dt} = -\mat{E}(\tilde{\mat{x}})
 $. Intuitively, theses ODEs uniquely define trajectories of particles between the $z=0$ hyperplane and an enclosing hemisphere (cf. \Figref{fig:heart-sphere}). In the following theorem, we show that the backward ODE defines a transformation between the uniform distribution on an infinite hemisphere and the data distribution $\tilde{p}(\tilde{\mat{x}})$ in the $z=0$ plane. We present the formal proof to Appendix~\ref{app:proofs}, illustrated by \Figref{fig:theorem1}. The proof is based on the idea that when the radius of hemisphere $r\to\infty$, the data distribution $\tilde{p}(\tilde{\mat{x}})$ can be effectively viewed as a delta distribution at origin. Consequently, the Poisson field points in the radial direction at $r\to\infty$, perpendicular to $S_N^+(r)$ (Green arrows in \Figref{fig:theorem1}).

\begin{restatable}{theorem}{sphere}
\label{thm:sphere} Suppose particles are sampled from a uniform distribution on the upper ($z>0$) half of the sphere of radius $r$ and evolved by the backward ODE $\frac{d\mat{\tilde{x}}}{dt}=-\mat{E}(\mat{\tilde{x}})$ until they reach the $z=0$ hyperplane, where the Poisson field $\mat{E}(\mat{\tilde{x}})$ is generated by the source $\tilde{p}(\tilde{\mat{x}})$. In the $r\to \infty$ limit, {under some mild conditions detailed in Appendix~\ref{app:proofs}}, this process generates a particle distribution $\tilde{p}(\tilde{\mat{x}})$, 
i.e., a distribution $p(\mat{x})$ in the $z=0$ hyperplane. 
\vspace{-5pt}
\end{restatable}
\begin{proofs} Suppose the flux of the backward ODE connects a solid angle $d\Omega$ (on $S_N^+(r)$) with an area $dA$ (on ${\rm supp}(\tilde{p}(\tilde{\mat{x}})$). According to Gauss's law, the outflow flux $d\Phi_{out}=d\Omega/S_N(1)$ on the hemisphere~(Green arrows in \Figref{fig:theorem1}) equals the inflow flux $d\Phi_{in}=p(\mat{x})dA/2$ on ${\rm supp}(\tilde{p}(\tilde{\mat{x}}))$~(Red arrows in \Figref{fig:theorem1}). $d\Phi_{in}=d\Phi_{out}$ gives $d\Omega/dA=p(\rvx)S_N(1)/2\propto p(\rvx)$. Together, by change-of-variable, we conclude that the final distribution in the $z=0$ hyperplane is $ p(\rvx)$.
\end{proofs}
The theorem states that starting from an infinite hemisphere, one can recover the data distribution $\tilde{p}$ by following the inverse Poisson field $-\mat{E}(\mat{\tilde{x}})$. {We defer the formal proof and technical assumptions of the theorem to Appendix~\ref{app:proofs}.} The property allows generative modeling by following the Poisson flow of $\nabla^2\varphi(\tilde{\rvx}) = -\tilde{p}(\tilde{\rvx})$.
\subsection{Learning the normalized Poisson Field}
\label{sec:learning}
Given a set of training data $\gD = \{\mat{x}_i\}_{i=1}^n$ i.i.d sampled from the data distribution $p(\mat{x})$, we define the empirical version of the Poisson field~(\Eqref{eq:gradient-field}) as follows:
$$
    \hat{\mat{E}}(\tilde{\mat{x}}) = c(\tilde{\mat{x}})\sum_{i=1}^n \frac{\tilde{\mat{x}}-\tilde{\mat{x}}_i}{||\tilde{\mat{x}}-\tilde{\mat{x}}_i||^{N+1}}
$$
where the gradient field is calculated on $n$ \textit{augmented} datapoints $\{\tilde{\mat{x}}_i=(\rvx_i,0)\}_{i=1}^n$, and $c(\tilde{\mat{x}})=1/\sum_{i=1}^n \frac{1}{||\tilde{\mat{x}}-\tilde{\mat{x}}_i||^{N+1}} $ is the multiplier for numerical stability. 
We further normalize the field to resolve the variations in the magnitude of the norm $\parallel  \hat{\mat{E}}(\tilde{\mat{x}}) \parallel_2$, and fit the neural network to the more amenable negative normalized field $\rvv(\tilde{\rvx}) = -\sqrt{N}{\hat{\mat{E}}(\tilde{\mat{x}})}/{\parallel \hat{\mat{E}}(\tilde{\mat{x}}) \parallel_2}$. The Poisson field is rescalable~(cf. Section~\ref{section:bg}) and thus trajectories of its forward/backward ODEs are invariant under normalization. We denote the empirical field calculated on batch data $\gB$ by $\hat{\mat{E}}_\gB$ and the negative normalized field as $\rvv_\gB(\tilde{\rvx}) = -\sqrt{N}{\hat{\mat{E}}_\gB(\tilde{\mat{x}})}/{\parallel \hat{\mat{E}}_\gB(\tilde{\mat{x}})\parallel_2} $. 
    
{Similar to the scored-based models, we sample points inside the hemisphere by perturbing the augmented training data. Given a training point $\rvx \in \gD$, we add noise to its augmented version $\{\tilde{\mat{x}}_i=(\rvx_i,0)\}_{i=1}^n$ to construct the perturbed point $(\rvy,z)$:
\begin{align*}
    \rvy = \rvx+\parallel \epsilon_\rvx \parallel(1+\tau)^m\rvu,\quad z=|\epsilon_z| (1+\tau)^m \numberthis \label{eq:geo-ode}
\end{align*}
where $\epsilon = (\epsilon_\rvx, \epsilon_z) \sim \gN(0, \sigma^2I_{N+1\times N+1})$, $\rvu\sim \gU(S_N(1))$ and $m\sim \gU[0,M]$. The upper limit $M$, standard deviation $\sigma$ and $\tau$ are hyper-parameters. With fixed $\epsilon$ and $\rvu$, the added noise increases exponentially with $m$. The rationale behind the design is that points farther away from the data support play a less important role in generative modeling, sharing a similar spirit with the choice of noisy scales in score-based models~\citep{ Song2020ImprovedTF,Song2021ScoreBasedGM}.}

In practice, we sample the points by perturbing a mini-batch data $\gB=\{\rvx_i\}_{i=1}^{|\gB|}$ in each iteration. We uniformly sample the power $m$ in $[0,M]$ for each datapoint. We select a large $M$ (typically around $300$) to ensure the perturbed points can reach a large enough hemisphere. We use a larger batch $\gB_L$ for the estimation of normalized field since the empirical normalized field is biased, which empirically gives better results. Denoting the set of perturbed points as $\{\tilde{\rvy}_i\}_{i=1}^{|\gB|}$, we train the neural network $f_\theta$ on these points to estimate the negative normalized field by minimizing the following loss:
\begin{align*}
    \gL(\theta) = \frac{1}{|\gB|}\sum_{i=1}^{|\gB|} \parallel f_\theta(\tilde{\rvy}_i) - \rvv_{\gB_L}(\tilde{\rvy}_i) \parallel_2^2
\end{align*}
We summarize the training process in Algorithm~\ref{alg:pf}. In practice, we add a small constant $\gamma$ to the denominator of the normalized field to overcome the numerical issue when $\exists i, ||\tilde{\mat{x}}-\tilde{\mat{x}}_i||\approx 0$.
\begin{algorithm}[htb]
  \caption{Learning the normalized Poisson Field}
  \label{alg:pf}
\begin{algorithmic}
  \STATE {\bfseries Input:} Training iteration $T$, Initial model $f_\theta$, dataset $\gD$, constant $\gamma$, learning rate $\eta$.
  \FOR{$t=1 \dots T$}
  \STATE Sample a large batch $\gB_L$ from $\gD$ and subsample a batch of datapoints $\gB = \{\rvx_i\}_{i=1}^{|\gB|}$ from $\gB_L$ 
      \STATE Simulate the ODE: $\{\tilde{\rvy}_i= $\textcolor{orange}{ \hspace{1.5pt}perturb$(\rvx_i)$} $\}_{i=1}^{|\gB|}$ 
      \STATE Calculate the normalized field by $\gB_L$: $\rvv_{\gB_L}(\tilde{\rvy}_i) = -\sqrt{N}\hat{\mat{E}}_{\gB_L}(\tilde{\rvy}_i)/(\parallel \hat{\mat{E}}_{\gB_L}(\tilde{\rvy}_i) \parallel_2+\gamma), \forall i$ 
      \STATE Calculate the loss: $\gL(\theta) = \frac{1}{|\gB|}\sum_{i=1}^{|\gB|} \parallel f_\theta(\tilde{\rvy}_i) - \rvv_{\gB_L}(\tilde{\rvy}_i) \parallel_2^2$
      \STATE Update the model parameter: $\theta = \theta - \eta \nabla \gL(\theta)$
  \ENDFOR
  \RETURN $f_\theta$
\end{algorithmic}
\end{algorithm}
\begin{algorithm}[htb]
  \caption{\textcolor{orange}{perturb$(\rvx)$}}
  \label{alg:ode}
\begin{algorithmic}
  \STATE Sample the power $m \sim \gU[0,M]$
  \STATE Sample the initial noise $(\epsilon_\rvx, \epsilon_z) \sim \gN(0, \sigma^2I_{(N+1)\times (N+1)})$
  \STATE Uniformly sample the vector from the unit ball $\rvu\sim \gU(S_N(1))$
    \STATE  Construct training point $\rvy = \rvx + \parallel \epsilon_\rvx \parallel(1+\tau)^m\rvu$, $z=|\epsilon_z| (1+\tau)^m$
  \RETURN $\tilde{\rvy}=(\rvy, z)$
\end{algorithmic}
\end{algorithm}

\subsection{Backward ODE anchored by the additional dimension}
\label{sec:sampling}
After estimating the normalized field $\rvv$, we can sample from the data distribution by the backward ODE $d \tilde{\rvx} =  \rvv(\tilde{\rvx}) dt$. Nevertheless, the boundary condition of the above ODE is unclear: the starting and terminal time $t$ of the ODE are both unknown. To remedy the issue, we propose an equivalent backward ODE in which $\rvx$ evolves with the augmented variable $z$:
\begin{align*}
    d(\rvx,z) = (\frac{d \rvx}{dt}\frac{d t}{dz}dz,dz) = (\rvv(\tilde{\rvx})_\rvx\rvv(\tilde{\rvx})_z^{-1}, 1) dz
\end{align*}
where $\rvv(\tilde{\rvx})_\rvx,\rvv(\tilde{\rvx})_z$ are the corresponding components of $\rvx,z$ in vector $\rvv(\tilde{\rvx})$. In the new ODE, we replace the time variable $t$ with the physically meaningful variable $z$, permitting explicit starting and terminal conditions: when $z=0$, we arrive at the data distribution and we can freely choose a large $\zmax$ as the starting point in the backward ODE. The backward ODE is compatible with general-purpose ODE solvers, \eg RK45 method~\citep{Lawrence1986SomePR} and forward Euler method. The popular black-box ODE solvers, such as the one in Scipy library~\citep{Virtanen2020SciPy1F}, typically use a common starting time for the same batch of samples. Since the distribution on the $z=\zmax$ hyperplane is no longer uniform, we derive the prior distribution by radially projecting uniform distribution on the hemisphere with radius $r=\zmax$ to the $z=\zmax$ hyperplane:
\begin{align*}
    p_{\textrm{prior}}(\rvx) =\frac{2\zmax^{N+1}}{S_N(\zmax)(\parallel \rvx \parallel_2^2 + \zmax^2)^{\frac{N+1}{2}}}=\frac{2\zmax}{S_N(1)(\parallel \rvx \parallel_2^2 + \zmax^2)^{\frac{N+1}{2}}}
\end{align*}
where $S_N(r)$ is the surface area of $N$-sphere with radius $r$. The reason behind the radial projection is that the Poisson field points in the radial direction at $r\to \infty$. The new backward ODE also defines a bijective transformation between $p_{\textrm{prior}}(\rvx)$ on the infinite hyperplane~($\zmax \to \infty$) and the data distribution $\tilde{p}(\tilde{\mat{x}})$, analogous to Theorem~\ref{thm:sphere}. In order to sample from $p_{\textrm{prior}}(\rvx)$, it is suffice to sample the norm~(radius) from the distribution:
$
    p_{\textrm{radius}}(\parallel \rvx\parallel_2) \propto {\parallel \rvx \parallel_2^{N-1}}/{(\parallel \rvx \parallel_2^2+\zmax^2)^{\frac{N+1}{2}}}
$
and then uniformly sample its angle. We provide detailed derivations and practical sampling procedure in Appendix \ref{app:prior_distribution}. We further achieve exponentially decay on the $z$ dimension by introducing a new variable $t'$:
\begin{align*}
{\rm[Backward\ ODE]} \quad {d(\rvx,z)} &= (\rvv(\tilde{\rvx})_\rvx\rvv(\tilde{\rvx})_z^{-1}z, z){dt'} \numberthis \label{eq:backode}
\end{align*}
The $z$ component in the backward ODE, \ie $dz = z dt'$, can be solved by $z=e^{t'}$. Since $z$ reaches zero as $t' \to -\infty$, we instead choose a tiny positive number $\zmin$ as the terminal condition. The corresponding starting/terminal time of the variable $t'$ are $\log \zmax/\log \zmin$ respectively. Empirically, this simple change of variable leads to $2\times$ faster sampling with almost no harm to the sample quality. In addition, we substitue the predicted $\rvv(\tilde{\rvx})_z$ with a more accurate one when $z$ is small~(Appendix~\ref{app:sub}). We defer more details of the simulation of backward ODE to Appendix~\ref{app:sampling}.

\section{Generative Modeling via the Backward ODE}

{In this section, we demonstrate the effectiveness of the backward ODE associated with PFGM on image generation tasks. In Section~\ref{sec:fids}, we show that PFGM achieves currently best in class performance in the normalizing flow family. In comparison to the existing state-of-the-art SDE or MCMC approaches, PFGM exhibits $10 \times$ or $20 \times$ acceleration while maintaining competitive or higher generation quality. Meanwhile, unlike existing ODE baselines that heavily rely on corrector to generate decent samples on weaker architectures, PFGM exhibits greater stability against error~(Section~\ref{sec:robust}).} Finally, we show that PFGM is robust to the step size in the Euler method~(Section~\ref{sec:adapt}), and its associated ODE allows for likelihood evaluation and image manipulation by editing the latent space~(Section~\ref{sec:utility}).

\subsection{Efficient image generation by PFGM}

\label{sec:fids}
\begin{table*}[]
    \small
    \centering
    \caption{CIFAR-10 sample quality~(FID, Inception) and number of function evaluation~(NFE).}
    \begin{tabular}{l c c c c}
    \toprule
        & Invertible? & Inception $\uparrow$  &FID $\downarrow$ & NFE $\downarrow$\\
         \midrule
        PixelCNN~\cite{Oord2016ConditionalIG} &\textcolor{red}{\xmark} & 4.60 &65.9 & 1024\\
        IGEBM~\citep{Du2019ImplicitGA}&\textcolor{red}{\xmark} & 6.02 &40.6 & 60\\
        ViTGAN~\cite{Lee2021ViTGANTG} & \textcolor{red}{\xmark} & $9.30$ &$6.66$& $1$\\
        StyleGAN2-ADA~\citep{Karras2020TrainingGA} & \textcolor{red}{\xmark} & $9.83$ & $2.92$ & $1$\\
      StyleGAN2-ADA~(cond.)~\citep{Karras2020TrainingGA}& \textcolor{red}{\xmark} & $10.14$ & $2.42$ & $1$\\
        NCSN~\citep{Song2019GenerativeMB}& \textcolor{red}{\xmark} & ${8.87 }$ & $25.32$ & $1001$\\
        NCSNv2~\citep{Song2020ImprovedTF}& \textcolor{red}{\xmark} & ${8.40 }$ & $10.87$ & $1161$\\
        DDPM~\citep{Ho2020DenoisingDP}& \textcolor{red}{\xmark}& $9.46$ & $3.17$&$1000$\\
        NCSN++ VE-SDE~\citep{Song2021ScoreBasedGM}& \textcolor{red}{\xmark}&  $9.83$&$2.38$ &$2000$\\
        NCSN++ deep VE-SDE~\citep{Song2021ScoreBasedGM}& \textcolor{red}{\xmark}&  $9.89$&$2.20$ &$2000$\\
        Glow~\cite{Kingma2018GlowGF}& \textcolor{green}{\cmark}&3.92&$48.9$& $1$\\
         DDIM, T=50~\citep{Song2021DenoisingDI} & \textcolor{green}{\cmark}&-&$4.67$& $50$\\
         DDIM, T=100~\citep{Song2021DenoisingDI} & \textcolor{green}{\cmark}&-&$4.16$& $100$\\
         NCSN++ VE-ODE~\citep{Song2021ScoreBasedGM}& \textcolor{green}{\cmark}&  $9.34$& $5.29$ & $194$\\
         NCSN++ deep VE-ODE~\citep{Song2021ScoreBasedGM}& \textcolor{green}{\cmark}&  $9.17$& $7.66$& $194$\\
        \midrule
        \textit{\textbf{DDPM++ backbone}}\\
        \midrule
        VP-SDE~\citep{Song2021ScoreBasedGM}& \textcolor{red}{\xmark} & $9.58$ &  $2.55$ &  $1000$\\
        sub-VP-SDE~\citep{Song2021ScoreBasedGM}& \textcolor{red}{\xmark} & $9.56$ & $2.61$ & $1000$\\
           \cdashlinelr{1-5}
         VP-ODE~\citep{Song2021ScoreBasedGM} & \textcolor{green}{\cmark}& $9.46$ &  $2.97$ &  $134$\\
         sub-VP-ODE~\citep{Song2021ScoreBasedGM}& \textcolor{green}{\cmark} & $9.30$ & $3.16$ & $146$\\
         PFGM~(ours) & \textcolor{green}{\cmark}& $\bm{9.65}$ & $\bm{2.48}$& $\bm{104}$\\
        \midrule
        \textit{\textbf{DDPM++ deep backbone}}\\
        \midrule
        VP-SDE~\citep{Song2021ScoreBasedGM}& \textcolor{red}{\xmark} & $9.68$ &  $2.41$ &  $1000$\\
        sub-VP-SDE~\citep{Song2021ScoreBasedGM}& \textcolor{red}{\xmark} & $9.57$ & $2.41$ & $1000$\\
        \cdashlinelr{1-5}
         VP-ODE~\citep{Song2021ScoreBasedGM} & \textcolor{green}{\cmark}&$9.47$& $2.86$& $134$\\
         sub-VP-ODE~\citep{Song2021ScoreBasedGM} & \textcolor{green}{\cmark}& $9.40$ & $3.05$&$146$\\
         PFGM~(ours)& \textcolor{green}{\cmark} & $\bm{9.68}$ & $\bm{2.35}$ & $\bm{110}$\\
         \bottomrule
    \end{tabular}
    \label{tab:cifar}
    \vspace{-2.5pt}
\end{table*}

\paragraph{Setup}{For image generation tasks, we consider the CIFAR-10~\citep{cifar}, CelebA $64\times 64$~\citep{Yang2015FromFP} and LSUN bedroom $256\times 256$~\citep{Yu2015LSUNCO}. Following \citep{Song2020ImprovedTF}, we first center-crop the CelebA images and then resize them to $64\times 64$. We choose $M=291~(\textrm{CIFAR-10 and CelebA})/356~(\textrm{LSUN bedroom})$, $\sigma=0.01$ and $\tau=0.03$ for the perturbation Algorithm~\ref{alg:ode}, and $\zmin=1e-3$, $\zmax=40~(\textrm{CIFAR-10})/60~(\textrm{CelebA $64^2$})/100~(\textrm{LSUN bedroom})$ for the backward ODE. We further clip the norms of initial samples into $(0, 3000)$ for CIFAR-10, $(0,6000)$ for CelebA $64^2$ and $(0, 30000)$ for LSUN bedroom. We adopt the DDPM++ and DDPM++ deep architectures~\cite{Song2021ScoreBasedGM} as our backbones. We add the scalar $z$ (resp. predicted direction on $z$) as input (resp. output) to accommodate the additional dimension. We take the same set of hyper-parameters, such as batch size, learning rate and training iterations from \cite{Song2021ScoreBasedGM}. We provide more training details in Appendix~\ref{app:training}, and discuss how to set these hyper-parameters for general datasets in \ref{app:hyper-train} and \ref{app:sample-add}.}

\paragraph{Baselines}We compare PFGM to modern autoregressive model~\citep{Oord2016ConditionalIG}, GAN~\citep{Karras2020TrainingGA,Lee2021ViTGANTG}, normalizing flow~\cite{Kingma2018GlowGF} and EBM~\citep{Du2019ImplicitGA}. We also compare with variants of score-based models such as DDIM~\cite{Song2021DenoisingDI} and current state-of-the-art SDE/ODE methods~\citep{Song2021ScoreBasedGM}. We denote the methods that use forward-time SDEs in \cite{Song2021ScoreBasedGM} such as Variance Exploding~(\textbf{VE}) SDE/Variance Preserving~(\textbf{VP}) SDE/ sub-Variance Preserving~(\textbf{sub-VP}), and the corresponding backward SDE/ODE, as \textbf{A-B}, where A $\in $ \{VE, VP, sub-VP\} and B $\in $ \{SDE, ODE\}. We follow the model selection protocol in \cite{Song2021ScoreBasedGM}, which selects the checkpoint with the smallest FID score over the course of training every 50k iterations.

\paragraph{Numerical Solvers}{The backward ODE~(\Eqref{eq:backode}) is compatible with any general purpose ODE solver. In our experiments, the default solver of ODEs is the black box solver in the Scipy library~\citep{Virtanen2020SciPy1F} with the RK45~\citep{Dormand1980AFO} method~(\textbf{RK45}), unless otherwise specified. For VE/VP/subVP-SDEs, we use the predictor-corrector~(\textbf{PC}) sampler introduced in \cite{Song2021ScoreBasedGM}. For VP/sub-VP-SDEs, we apply the predictor-only sampler, because its performance is on par with the PC sampler while requiring half computation.}

\paragraph{Results}{For quantitative evaluation on CIFAR-10, we report the Inception~\citep{Salimans2016ImprovedTF}~(higher is better) and FID~\citep{Heusel2017GANsTB} scores~(lower is better) in Table~\ref{tab:cifar}. We also include our preliminary experimental results on a weaker architecture NCSNv2~\cite{Song2020ImprovedTF} in Appendix~\ref{app:ncsnv2}. We measure the inference speed by the average NFE (number of function evaluation). We also explicitly indicate which methods belong to the invertible flow family.}

{Our main findings are: \textbf{(1)} \textbf{PFGM achieves the best Inception scores and FID scores among the normalizing flow models.} Specifically, PFGM obtains an Inception score of $9.68$ and a FID score of $2.48$ using the DDPM++ deep architecture. To our best knowledge, these are the highest FID and Inception scores by flow models on CIFAR-10. \textbf{(2)} \textbf{PFGM achieves a $10\times \sim 20 \times$ faster inference speed than the SDE methods using the similar architectures, while retaining comparable sample quality.} As shown in Table~\ref{tab:cifar}, PFGM requires NFEs of 110 whereas the SDE methods typically use $1000\sim 2000$ inference steps. PFGM outperforms all the baselines on DDPM++ in all metrics. In addition, PFGM generally samples faster than other ODE baselines with the same RK45 solver. \textbf{(3)} \textbf{The backward ODE in PFGM is compatible with architectures with varying capacities.} PFGM consistently outperforms other ODE baselines on DDPM++~(Table~\ref{tab:cifar}) or NCSNv2~(Appendix~\ref{app:ncsnv2}) backbones. \textbf{(4)} \textbf{PFGM shows scalability to higher resolution datasets.} In Appendix~\ref{app:exp-lsun}, we show that PFGM are capable of scale-up to LSUN bedroom $256\times 256$. In particular, PFGM has comparable performance with VE-SDE with 15$\times$ fewer NFE.  }

In \Figref{img:vis}, we visualize the uncurated samples from PFGM on CIFAR-10, CelebA $64 \times 64$ and LSUN bedroom $256\times 256$. We provides more samples in Appendix~\ref{app:samples}.
\begin{figure*}[t]
    \centering
    \includegraphics[width=1.0\textwidth]{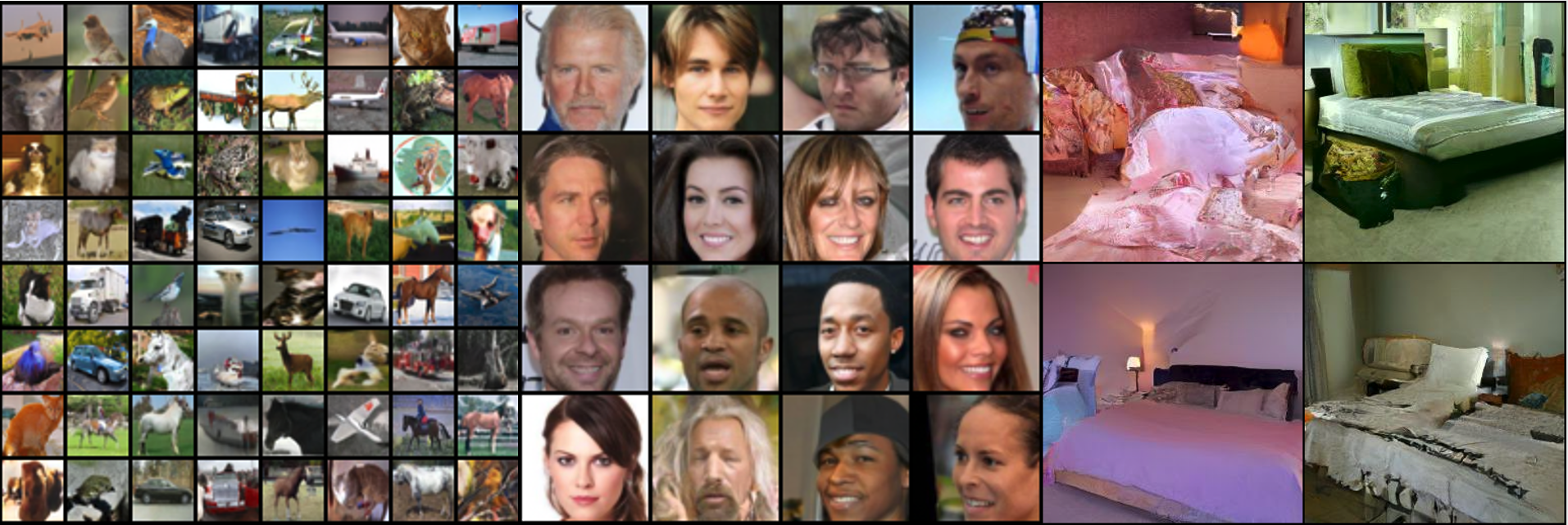}
    \caption{Uncurated samples on datasets of increasing resolution. From left to right: CIFAR-10 $32\times 32$, CelebA $64\times 64$ and LSUN bedroom $256\times 256$.} \label{img:vis}
    \vspace{-5pt}
\end{figure*}
\subsection{Failure of VE/VP-ODEs on NCSNv2 architecture}
\label{sec:robust}
\begin{wrapfigure}{r}{0.5\textwidth} 
\vspace{-20pt}
  \begin{center}
    \includegraphics[width=0.5\textwidth]{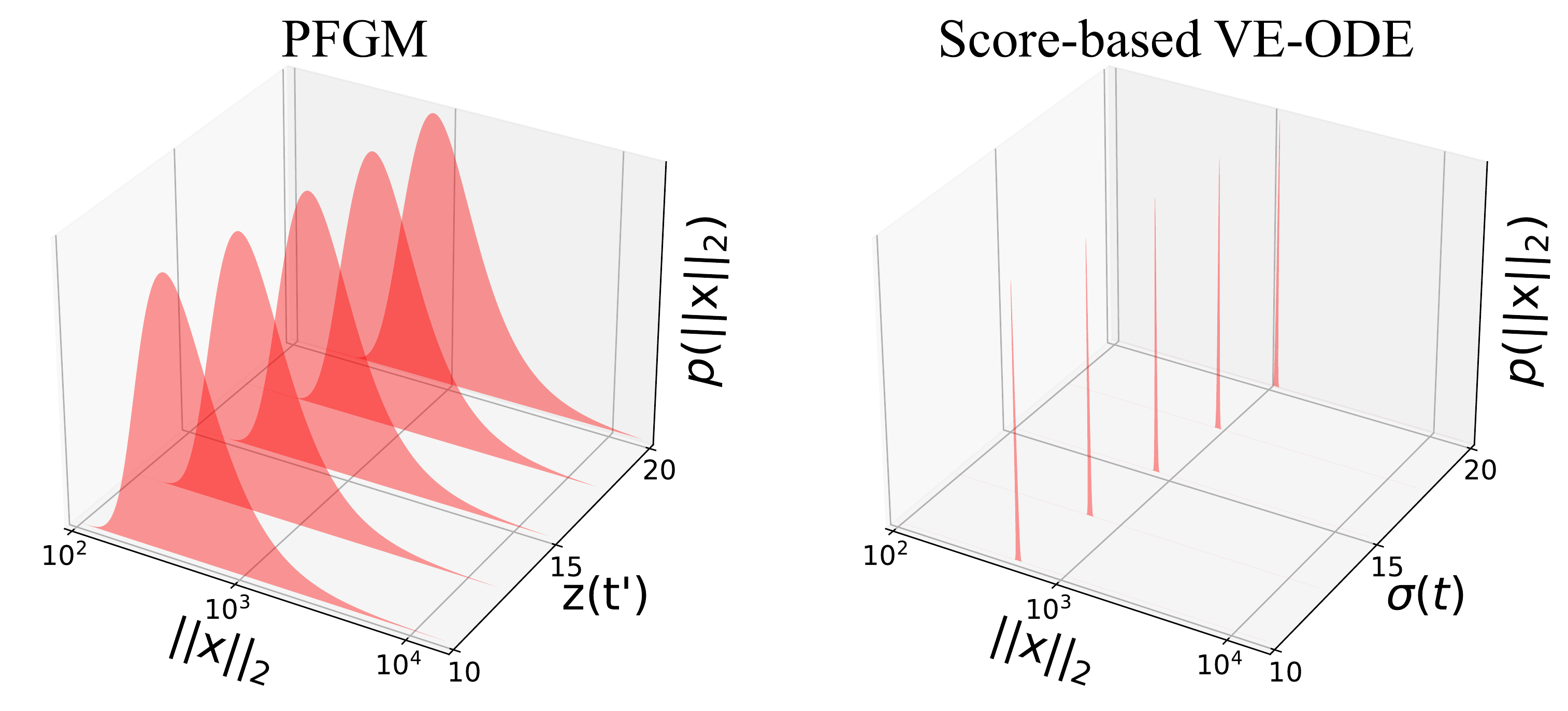}
    \caption{Sample norm distributions with varying time variables~($\sigma$ for VE-ODE and $z$ for PFGM)}
    \label{fig:compare_prior}
  \end{center}
  \vspace{-20pt}
\end{wrapfigure}


 
In our preliminary experiments on NCSNv2 architectures, we empirically observe that the VE/VP-ODEs have FID scores greater than 90 on CIFAR-10. In particular, VE/VP-ODEs can only generate decent samples when applying the Langevin dynamics corrector, and even then, their performances are still inferior to PFGM~(Table~\ref{tab:cifar-ncsnv2}, Table~\ref{tab:celeba}). The poor performance on NCSNv2 stands in striking contrast to their high sample quality on NCSN++/DDPM++ in \cite{Song2021ScoreBasedGM}. \textbf {It indicates that the VE/VP-ODEs are more susceptible to estimation errors than PFGM.} We hypothesize that the strong norm-$\sigma$ correlation seen during the training of score-based models causes the problem.

For score-based models, the $l_2$ norms of perturbed training samples and the standard deviations $\sigma(t)$ of Gaussian noises have strong correlation, \eg $l_2$ norm $\approx \sigma(t)\sqrt{N}$ for large $\sigma(t)$ in VE~\citep{Song2021ScoreBasedGM}. In contrast, as shown in \Figref{fig:compare_prior}, PFGM allocates high mass across a wide spectrum of the training sample norms. During sampling, VE/VP-ODEs could break down when the trajectories of backward ODEs deviate from the norm-$\sigma(t)$ relation to which most training samples pertain. The weaker NCSNv2 backbone incurs larger errors and thus leads to their failure. The PFGM is more resistant to estimate errors because of the greater range of training sample norms.

{To further verify the hypothesis above, we split a batch of VE-ODE samples into cleaner and noisier samples according to visual quality~(\Figref{fig:failure-a}). In \Figref{fig:ve}, we investigate the relation for cleaner and noisier samples during the forward Euler simulation of VE-ODE when $\sigma(t)<15$. We can see that the trajectory of cleaner samples stays close to the norm-$\sigma(t)$ relation~(the red dash line), whereas that of the noisier samples diverges from the relation. The Langevin dynamics corrector changes the trajectory of noisier samples to align with the relation. 
\Figref{fig:pfgm-norm} further shows that the anchored variable $z(t')$ and the norms in the backward ODE of PFGM are not strongly correlated, giving rise to the robustness against the imprecise estimation on NCSNv2. We defer more details to Appendix~\ref{app:failure}.}

\subsection{{Effects of step size in the forward Euler method}}
\label{sec:adapt}
\begin{figure*}[t]
    \centering
        \subfigure[Norm-$\sigma(t)$ in VE-ODE]{ \label{fig:ve}
            \includegraphics[width=0.31\textwidth]{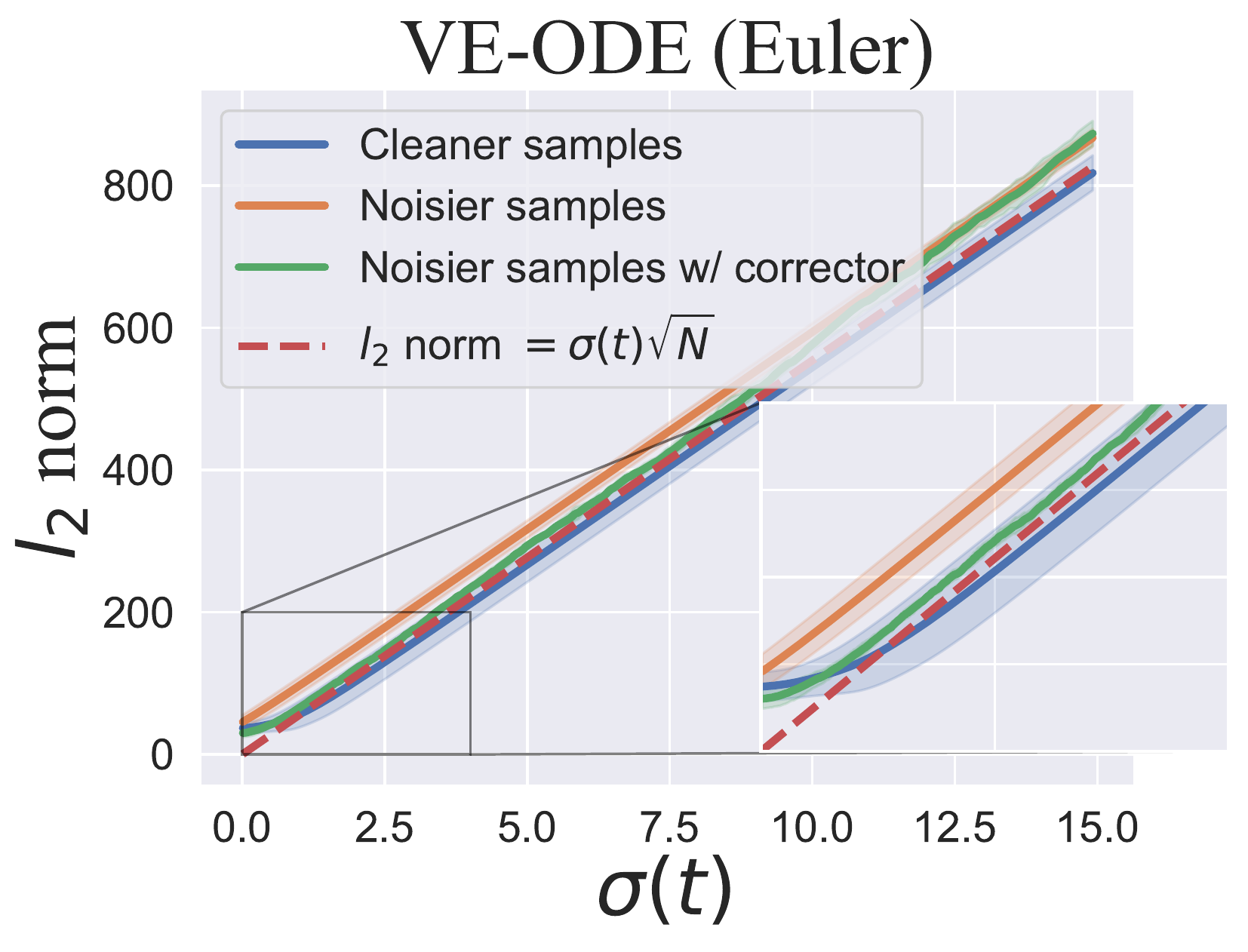}}
        \subfigure[Norm-$z(t')$ in PFGM]{\label{fig:pfgm-norm}
        \includegraphics[width=0.29\textwidth]{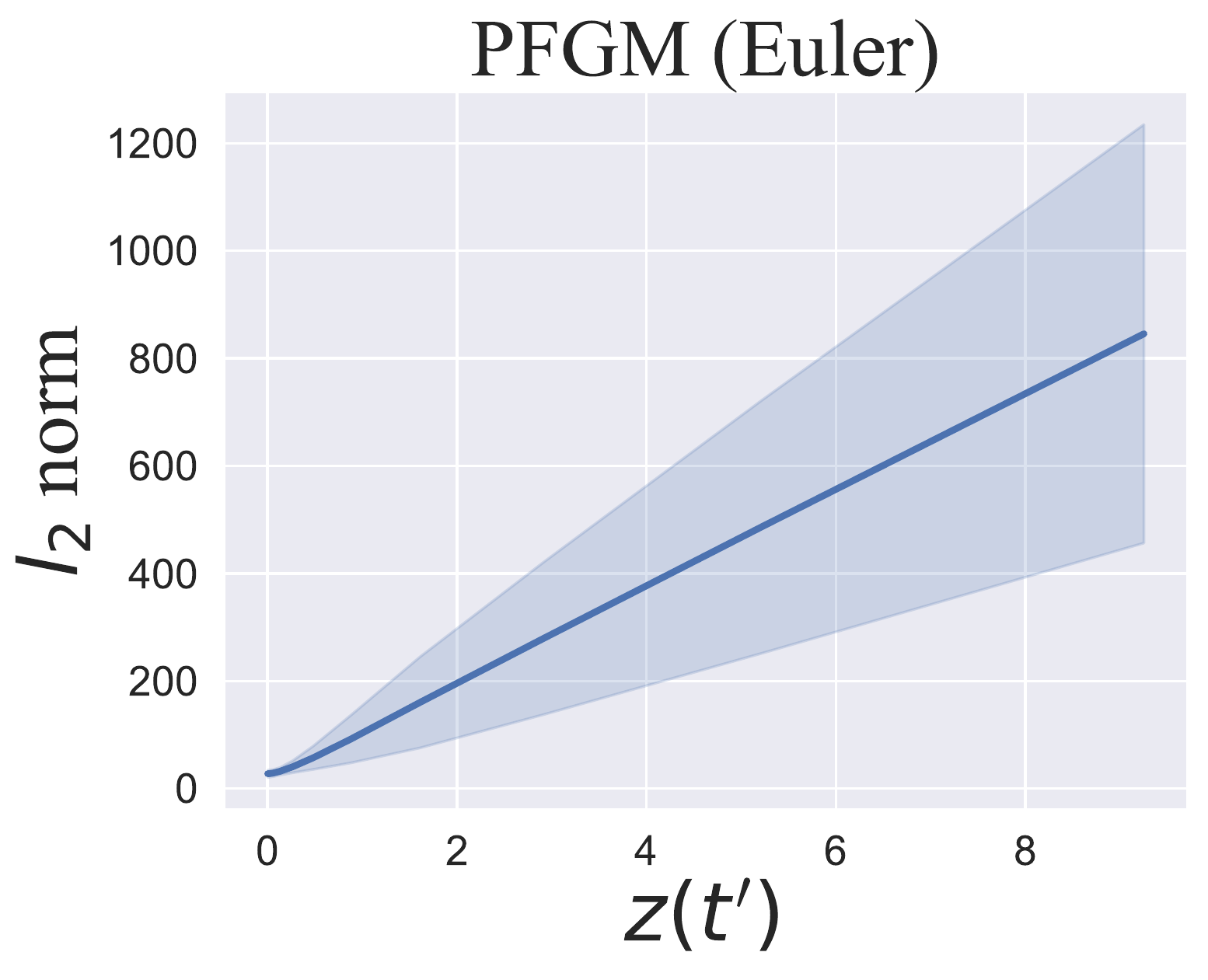}}
      \centering
    \subfigure[FID vs. NFE on CIFAR-10]{\label{fig:adapt}\includegraphics[width=0.31\textwidth]{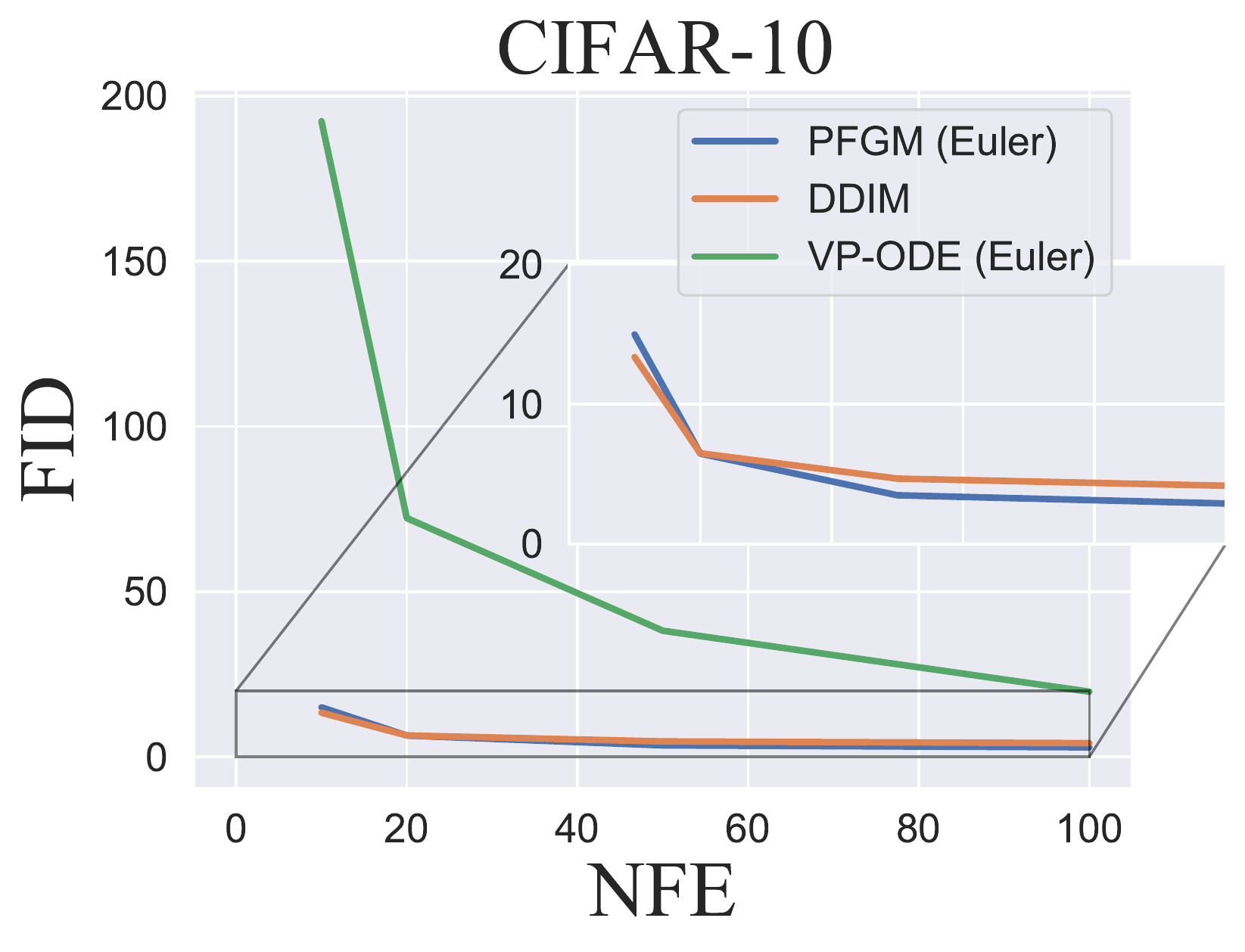}}
    \caption{\textbf{(a)} Norm-$\sigma(t)$ relation during the backward sampling of VE-ODE (Euler). \textbf{(b)} Norm-$z(t')$ relation during the backward sampling of PFGM (Euler). The shaded areas mean the standard deviation of norms. \textbf{(c)} Number of steps versus FID score.}
    \vspace{-5pt}
\end{figure*}
{In order to accelerate the inference speed of ODEs, we can increase the step size~(decrease the NFEs) in numerical solvers such as the forward Euler method. It also enables the trade-off between sample quality and computational efficiency in real-world deployment. We study the effects of increasing step size on PFGM, VP-ODE and DDIM~\cite{Song2021DenoisingDI} using the forward Euler method, with a varying NFE ranging from $10$ to $100$.} 

{In \Figref{fig:adapt}, we report the sample quality measured by FID scores on CIFAR-10. As expected, all the methods have higher FID scores when decreasing the NFE. We observe that the sample quality of PFGM degrades gracefully as we decrease the NFE. Our method shows significantly better robustness to step sizes than the VP-ODE, especially when only taking a few Euler steps. In addition, PFGM obtains better FID scores than DDIM on most NFEs except for $10$ where PFGM is marginally worse. This suggests that the PFGM is a promising method for accommodating instantaneous resource availability, as high-quality samples can be generated in limited steps.}


\subsection{Utilities of ODE: likelihood evaluation and latent representation}
\label{sec:utility}
Similar to the family of discrete normalizing flows~\citep{Dinh2017DensityEU, Kingma2018GlowGF, Ho2019FlowIF} and continuous probability flow~\citep{Song2021ScoreBasedGM}, the forward ODE in PFGM defines an invertible mapping between the data space and latent space with a known prior. Formally, we define the invertible forward $\mathcal{M}$ mapping by integrating the corresponding forward ODE ${d(\rvx,z)} = (\rvv(\tilde{\rvx})_\rvx\rvv(\tilde{\rvx})_z^{-1}z, z){dt'}$ of \Eqref{eq:backode}:
\begin{align*}
        {\mat{x}}(\log{{\zmax}})&=\mathcal{M}({\mat{x}}(\log {\zmin})) \equiv {\mat{x}}(\log {\zmin})+\int_{\log {\zmin}}^{\log{{\zmax}}} \rvv({\rvx}(t'))_\rvx\rvv(\tilde{\rvx}(t'))_z^{-1}e^{t'} dt' 
\end{align*}
where $\log{\zmin}$/$\log{{\zmax}}$ are the starting/terminal time in the forward ODE. The forward mapping transfers the data distribution to the prior distribution $p_{\textrm{prior}}$ on the $z=\zmax$ hyperplane~(cf. Section~\ref{sec:sampling}):
$
    p_{\textrm{prior}}({\mat{x}}(\log{{\zmax}})) = \mathcal{M}(p({\mat{x}}(\log {\zmin})))
$. The invertibility enables likelihood evaluation and creates a meaningful latent space on the $z=\zmax$ hyperplane. In addition, we can adapt to the computational constraints by adjusting the step size or the precision in numerical ODE solvers.

\begin{wrapfigure}{r}{0.4\textwidth} 
\centering
\scriptsize
    \vspace{-12pt}
\captionof{table}{Bits/dim on CIFAR-10}\label{tab:likelihood}
\begin{tabular}{l c }
    \toprule
         & bits/dim $\downarrow$\\
         \midrule
         RealNVP~\citep{Dinh2017DensityEU} & 3.49 \\
    Glow~\citep{Kingma2018GlowGF} & 3.35  \\
    Residual Flow~\citep{Chen2019ResidualFF} & 3.28 \\
    Flow++~\citep{Ho2019FlowIF} & 3.29 \\
    DDPM ($L$)~\citep{Ho2020DenoisingDP} & $\leq$ 3.70\textsuperscript{*} \\
    \midrule
    \textit{DDPM++ backbone}\\
    \midrule
    VP-ODE~\citep{Song2021ScoreBasedGM} & 3.20 \\
    sub-VP-ODE~\citep{Song2021ScoreBasedGM} & \bf{3.02} \\
        PFGM~(ours)& 3.19 \\
         \bottomrule
         \vspace{-15pt}
    \end{tabular}
\end{wrapfigure}\textbf{Likelihood evaluation} We evaluate the data likelihood by the instantaneous change-of-variable formula~\citep{Chen2018NeuralOD, Song2021ScoreBasedGM}. In Table~\ref{tab:likelihood}, we report the bits/dim on the uniformly dequantized CIFAR-10 test set and compare with existing baselines that use the same setup.
We observe that PFGM achieves better likelihoods than discrete normalizing flow models, even without maximum likelihood training. Among the continuous flow models, sub-VP-ODE shows the lowest bits/dim, although its sample quality is worse than VP-ODE and PFGM (Table~\ref{tab:cifar}). The exploration of the seeming trade-off between likelihood and sample quality is left for future works.


\textbf{Latent representation} {Since the samples are uniquely identifiable by their latents via the invertible mapping $\mathcal{M}$, PFGM further supports image manipulation using its latent representation on the $z=\zmax$ hyperplane. We include the results of image interpolation and the temperature scaling~\citep{Dinh2017DensityEU, Kingma2018GlowGF, Song2021ScoreBasedGM} to Appendix~\ref{app:interpolate} and Appendix~\ref{app:temp}. For interpolation, it shows that we can travel along the latent space to obtain perceptually consistent interpolations between CelebA images.}


\section{Conclusion}

\label{sec:conclusion}

We present a new deep generative model by solving the Poisson equation whose source term is the data distribution. We estimate the normalized gradient field of the solution in an augmented space with an additional dimension. For sampling, we devise a backward ODE that exponential decays on the physically meaningful additional dimension. Empirically, our approach has currently best performance over other normalizing flow baselines, and achieving $10\times$ to $20 \times$ acceleration over the stochastic methods. Our backward ODE shows greater stability against errors than popular ODE-based methods, and enables efficient adaptive sampling. We further demonstrate the utilities of the forward ODE on likelihood evaluation and image interpolation. Future directions include improving the {normalization of Poisson fields}. More principled approaches can be used to get around the divergent near-field behavior. For example, we may exploit renormalization, a useful tool in physics, to make the Poisson field well-behaved in near fields. 

\section*{Acknowledgements}

We are grateful to Shangyuan Tong, Timur Garipov and Yang Song for helpful discussion. We would like to thank Octavian Ganea and Wengong Jin for reviewing an early draft of this paper. YX and TJ acknowledge support from MIT-DSTA Singapore collaboration, from NSF Expeditions grant (award 1918839) "Understanding the World Through Code", and from MIT-IBM Grand Challenge project. ZL and MT would like to thank the Center for Brains, Minds, and Machines (CBMM)
for hospitality. ZL and MT are supported by The Casey and Family Foundation, the Foundational Questions Institute, the Rothberg Family Fund for Cognitive Science and IAIFI through NSF grant PHY-2019786.
\clearpage
\bibliography{ref}
\bibliographystyle{plain}
\clearpage

\appendix
\newpage

{\huge Appendix}

\def\E{{\bf E}}
\def\x{{\bf x}}
\def\r{{\bf r}}
\def\rhat{\hat{r}}

\section{Proofs}
\label{app:proofs}
\subsection{Formal Proof of Theorem 1}

Before proceeding to Theorem 1, we show a technical lemma that guarantees the existence-uniqueness of the solution to the Poisson equation, under some mild conditions. 
\begin{lemma}
\label{lemma:ex-uni}
{Given $\Omega=\mathbb{R}^N, N\ge 3$, assume that the source function $\rho \in \gC^0(\Omega)$, and $\rho$ has a compact support. Then the the Poisson equation $\nabla^2\varphi(\mat{x})=-\rho(\mat{x})$ on $\Omega$ with zero boundary condition at infinity~($\lim_{\parallel \rvx \parallel_2 \to \infty}\varphi(\mat{x})=0$) has a unique solution $\varphi(\mat{x})\in \gC^2(\Omega)$ up to a constant.}
\end{lemma}

\begin{proof}

For the existence of the solution, one can verify that the analytical construction using the extension of Green's function in $N\ge 3$ dimensional space~(Lemma~\ref{lemma:green}), \ie $
    \varphi(\mat{x}) = \int G(\mat{x},\mat{y})\rho(\mat{y})d\mat{y},   G(\mat{x},\mat{y}) = \frac{1}{(N-2)S_{N-1}(1)} \frac{1}{||\mat{x}-\mat{y}||^{N-2}}
$, is one possible solution to the Poisson equation $\nabla^2\varphi(\mat{x})=-\rho(\mat{x})$. Since $\rho\in \gC^0(\Omega)$ and $\nabla^2\varphi(\mat{x})=-\rho(\mat{x})$, we conclude that $\varphi(\mat{x})\in \gC^2(\Omega)$.

The proof idea of the uniqueness is similar to the uniqueness theorems in electrostatics. Suppose we have two different solutions $\varphi_1,\varphi_2\in \gC^2$ which satisfy 
\begin{equation}
    \nabla^2\varphi_1(\mat{x}) = -\rho(\mat{x}), \nabla^2\varphi_2(\mat{x}) = -\rho(\mat{x}).
\end{equation}
We define $\tilde{\varphi}(\mat{x})\equiv\varphi_2(\mat{x})-\varphi_1(\mat{x})$. Subtracting the above two equations gives
\begin{equation}\label{eq:unique_1}
    \nabla^2\tilde{\varphi}(\mat{x}) = 0, \forall \rvx \in \Omega.
\end{equation}
By the  vector differential identity we have
\begin{equation}\label{eq:unique_2}
    \tilde{\varphi}(\mat{x})\nabla^2\tilde{\varphi}(\mat{x}) = \nabla\cdot(\tilde{\varphi}(\mat{x})\nabla\tilde{\varphi}(\mat{x}))-\nabla\tilde{\varphi}(\mat{x})\cdot \nabla\tilde{\varphi}(\mat{x}),
\end{equation}
By the divergence theorem we have
\begin{equation}\label{eq:unique_3}
    \int_{\Omega} \nabla\cdot(\tilde{\varphi}(\mat{x})\nabla\tilde{\varphi}(\mat{x})) d^N\mat{x} = \oiint_{\partial\Omega} \tilde{\varphi}(\mat{x})\nabla\tilde{\varphi}(\mat{x})\cdot d^{N-1}\mat{S} = 0,
\end{equation}
where $d^{N-1}\mat{S}$ denotes an $N-1$ dimensional surface element at infinity, and the second equation holds due to zero boundary condition at infinity. Combining Eq.~(\ref{eq:unique_1})(\ref{eq:unique_2})(\ref{eq:unique_3}), we have
\begin{equation}
    \int_{\Omega} \nabla\cdot(\tilde{\varphi}(\mat{x})\nabla\tilde{\varphi}(\mat{x})) d^N\mat{x}=\int_{\Omega} ||\nabla\tilde{\varphi}(\mat{x})||^2 d^N\mat{x}=0,
\end{equation}
since this is an integral of a positive quantity, we must have $\nabla\tilde{\varphi}(\mat{x})=\mat{0}$, or $\tilde{\varphi}(\mat{x})=c$, $\forall\mat{x}\in\Omega$. This means $\varphi_1$ and $\varphi_2$ differ at most by a constant, but a constant does not affect gradients, so $\nabla\varphi_1(\mat{x})=\nabla\varphi_2(\mat{x})$. 
\end{proof}

{In our method section~(Section~\ref{sec:augment}), we augmented the original $N$-dimensional data with an extra dimension. The new data distribution in the augmented space is $\tilde{p}(\tilde{\mat{x}})=p(\mat{x})\delta(z)$, where $\delta$ is the Dirac delta function. The support of the data distribution is in the $z=0$ hyperplane. In the following lemma, we show the existence and uniqueness of the solution to $\nabla^2\varphi(\tilde{\rvx})=-\tilde{p}(\tilde{\rvx})$ outside the data support.}

\begin{lemma}
\label{lemma:data}
{Assume the support of the data distribution in the augmented space~($\textrm{supp}(\tilde{p}(\tilde{\rvx}))$) is a compact set on the $z=0$ hyperplane, $p(\rvx) \in \gC^0$ and  $N\ge 3$. The Poisson equation $\nabla^2\varphi(\tilde{\rvx})=-\tilde{p}(\tilde{\rvx})$ with zero boundary condition at infinity~($\lim_{\parallel \rvx \parallel_2 \to \infty}\varphi(\tilde{\rvx})=0$) has a unique solution $\varphi(\tilde{\rvx}) \in \gC^2$ for $\tilde{x} \in \mathbb{R}^{N+1} \setminus \textrm{supp}(\tilde{p}(\tilde{\rvx}))$, up to a constant.}
\end{lemma}
\begin{proof}

Similar to the proof in Lemma~\ref{lemma:ex-uni}, one can easily verify that the analytical construction using Green's method, \ie $\varphi(\tilde{\rvx}) = \int G(\tilde{\rvx},\tilde{\rvy})\tilde{p}(\tilde{\rvx})d\tilde{\rvy}, G(\tilde{\rvx},\tilde{\rvy}) = \frac{1}{(N-1)S_{N}(1)} \frac{1}{||\tilde{\rvx}-\tilde{\rvy}||^{N-1}}$, is one possible solution to the Poisson equation $\nabla^2\varphi(\tilde{\rvx})=-\tilde{p}(\tilde{\rvx})$. Since $\tilde{p}(\tilde{\rvx})=0$ for  $\tilde{\rvx} \in \mathbb{R}^{N+1} \setminus \textrm{supp}(\tilde{p}(\tilde{\rvx}))$ and $\nabla^2\varphi(\tilde{\rvx})=-\tilde{p}(\tilde{\rvx})$, we conclude that $\varphi(\tilde{\rvx})\in \gC^2(\mathbb{R}^{N+1} \setminus \textrm{supp}(\tilde{p}(\tilde{\rvx})))$.

For the uniqueness, suppose we have two different solutions $\varphi_1,\varphi_2\in \gC^2(\mathbb{R}^{N+1} \setminus \textrm{supp}(\tilde{p}(\tilde{\rvx})))$ which satisfy 
\begin{equation}
    \nabla^2\varphi_1(\tilde{\rvx}) = -\tilde{p}(\tilde{\rvx}), \nabla^2\varphi_2(\tilde{\rvx}) = -\tilde{p}(\tilde{\rvx}).
\end{equation}
We define $\tilde{\varphi}(\tilde{\rvx})\equiv\varphi_2(\tilde{\rvx})-\varphi_1(\tilde{\rvx})$. Subtracting the above two equations gives
\begin{equation}\label{eq:2unique_1}
    \nabla^2\tilde{\varphi}(\tilde{\rvx}) = 0, \forall \tilde{\rvx} \in \mathbb{R}^{N+1} \setminus \textrm{supp}(\tilde{p}(\tilde{\rvx})).
\end{equation}
By the vector differential identity we have
\begin{equation}\label{eq:2unique_2}
    \tilde{\varphi}(\tilde{\rvx})\nabla^2\tilde{\varphi}(\tilde{\rvx}) = \nabla\cdot(\tilde{\varphi}(\tilde{\rvx})\nabla\tilde{\varphi}(\tilde{\rvx}))-\nabla\tilde{\varphi}(\tilde{\rvx})\cdot \nabla\tilde{\varphi}(\tilde{\rvx}),
\end{equation}
By the divergence theorem we have
\begin{equation}\label{eq:2unique_3}
    \int_{\mathbb{R}^{N+1}} \nabla\cdot(\tilde{\varphi}(\tilde{\rvx})\nabla\tilde{\varphi}(\tilde{\rvx})) d^{N+1}\tilde{\rvx} = \oiint_{\partial\mathbb{R}^{N+1}} \tilde{\varphi}(\tilde{\rvx})\nabla\tilde{\varphi}(\tilde{\rvx})\cdot d^{N}\mat{S} = 0,
\end{equation}
where $d^{N}\mat{S}$ denotes an $N$ dimensional surface element at infinity, and the second equation holds due to zero boundary condition at infinity. Combining Eq.~(\ref{eq:2unique_1})(\ref{eq:2unique_2})(\ref{eq:2unique_3}), we have
\begin{align*}
    \int_{\mathbb{R}^{N+1}} \nabla\cdot(\tilde{\varphi}(\tilde{\rvx})\nabla\tilde{\varphi}(\tilde{\rvx})) d^{N+1}\tilde{\rvx} &= \int_{\mathbb{R}^{N+1} \setminus \textrm{supp}(\tilde{p}(\tilde{\rvx}))} \nabla\cdot(\tilde{\varphi}(\tilde{\rvx})\nabla\tilde{\varphi}(\tilde{\rvx})) d^{N+1}\tilde{\rvx}\\
    &=\int_{\mathbb{R}^{N+1} \setminus \textrm{supp}(\tilde{p}(\tilde{\rvx}))} ||\nabla\tilde{\varphi}(\tilde{\rvx})||^2 d^{N+1}\tilde{\rvx}=0,
\end{align*}
The first equation holds because Lebesgue measure of $\textrm{supp}(\tilde{p}(\tilde{\rvx}))$ is zero. Since $||\nabla\tilde{\varphi}(\tilde{\rvx})||^2$ is an integral of a positive quantity, we must have $\nabla\tilde{\varphi}(\tilde{\rvx})=\mat{0}$, or $\tilde{\varphi}(\tilde{\rvx})=c$, $\forall\tilde{\rvx}\in\mathbb{R}^{N+1}\setminus \textrm{supp}(\tilde{p}(\tilde{\rvx}))$. This means $\varphi_1$ and $\varphi_2$ differ at most by a constant function, but a constant does not affect gradients, so $\nabla\varphi_1(\tilde{\rvx})=\nabla\varphi_2(\tilde{\rvx})$. 
\end{proof}

\begin{figure*}
    \centering
    \includegraphics[width=0.6\textwidth, trim=0cm 3cm 0cm 3cm]{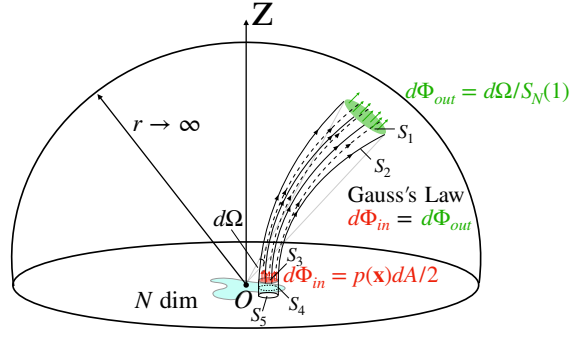}
    \caption{Proof idea of Theorem~\ref{theorem2}. By Gauss's Law, the outflow flux $d\Phi_{out}$ equals the inflow flux $d\Phi_{in}$. The factor of two in $p(\rvx)dA/2$ is due to the symmetry of Poisson fields in $z<0$ and $z>0$.
    }
        \label{fig:theorem2}
\end{figure*}

As illustrated in \Figref{fig:theorem2}, there is a bijective mapping between the upper hemisphere of radius $r$ and the $z=0$ plane, where each pair of corresponding points is connected by an electric field line. 
We will now formally prove that, in the $r\to\infty$ limit, this mapping transforms the arbitrary charge distribution in the source plane (that generated the electric field) into a uniform distribution on the hemisphere. 
\begin{theorem}
\label{theorem2}
 {Suppose particles are sampled from a uniform distribution on the upper ($z>0$) half of the sphere of radius $r$ and evolved by the backward ODE $\frac{d\mat{\tilde{x}}}{dt}=-\mat{E}(\mat{\tilde{x}})$ until they reach the $z=0$ hyperplane, where the Poisson field $\mat{E}(\mat{\tilde{x}})$ is generated by the source $\tilde{p}(\tilde{\mat{x}})$. In the $r\to \infty$ limit, {under the conditions in Lemma~\ref{lemma:data}}, this process generates a particle distribution $\tilde{p}(\tilde{\mat{x}})$, 
i.e., a distribution $p(\mat{x})$ in the $z=0$ hyperplane.} 
\end{theorem}
\begin{proof}

{By Lemma~\ref{lemma:data}, we know that with zero boundary at infinity, the Poisson equation $\nabla^2\varphi(\tilde{\rvx})=-\tilde{p}(\tilde{\rvx})$ has a unique solution $\varphi(\tilde{\rvx}) \in \gC^2$ for $\tilde{\rvx} \in \mathbb{R}^{N+1} \setminus \textrm{supp}(\tilde{p}(\tilde{\rvx}))$. Hence $\mat{E}(\mat{\tilde{x}})= -\nabla \varphi(\tilde{\rvx}) \in \gC^1$, guaranteeing the existence-uniqueness of the solution to the ODE $\frac{d\mat{\tilde{x}}}{dt}=-\mat{E}(\mat{\tilde{x}})$ according to Theorem 2.8.1 in \cite{Ricardo2002AMI}.}

Consider the tube in \Figref{fig:theorem2} connecting an area on $dA$ in the $z=\epsilon\to 0^+$ hyperplane ($S_3$) to a solid angle $d\Omega$ on the hemisphere ($S_1$), with $S_2$ as its side. The tube is the space swept by $dA$ following electric field $\mat{E}$, so by definition the electric field is parallel to the tangent space of the tube sides $S_2$. The bottom of the tube $S_3$ is located at $z=\epsilon\to 0^+$, a bit above the $z=0$ plane, so the tube does not enclose any charges.  We note that the divergence of Poisson field is zero in $\mathbb{R}^{N+1} \setminus \textrm{supp}(\tilde{p}(\tilde{\rvx}))$:
\begin{align*}
    \nabla\cdot \mat{E}(\mat{\tilde{x}}) = -\nabla^2\varphi(\tilde{\rvx})=\tilde{p}(\tilde{\rvx})=0, \forall \tilde{\rvx} \in \mathbb{R}^{N+1} \setminus \textrm{supp}(\tilde{p}(\tilde{\rvx}))
\end{align*}
Denote the volume and surface of the tube as $V$ and $\mat{B}$. According to divergence theorem, $ \oiint \mat{E}(\mat{\tilde{x}})\cdot d\mat{B}=\int_V \nabla\cdot \mat{E}(\mat{\tilde{x}}) dV=0$. Hence the net flux leaving the tube is zero:
\begin{equation}\label{eq:tube_zeronetflux}
    \Phi_{S_1} + \Phi_{S_2} + \Phi_{S_3} = 0, \quad
    \Phi_{S_i} \equiv \oiint_{S_i} \mat{E}(\mat{\tilde{x}})\cdot d\mat{B} \quad (i=1,2,3)
\end{equation}

There is no flux through the sides, i.e., $\Phi_{S_2}=0$, since $\mat{E}(\mat{\tilde{x}})$ is orthogonal to the surface element $d\mat{B}$ on the tube sides by definition. As a result, the flux $\Phi_{S_3}$ entering from below must equal the flux $\Phi_{S_1}$ leaving the other end. Denote the $l_2$ norm of the vector $\r$ as $r$. We first calculate the influx $\Phi_{S_3}$. To do so, we study a Gaussian pillbox whose top, side and bottom are $S_3$, $S_4$ and $S_5$. $S_3$ and $S_5$ are located at $z=\epsilon$ and $z=-\epsilon$ ($\epsilon\to 0^+$). Denote the volume and surface of the pillbox as $V'$ and $\mat{B}'$. The pillbox contains charge $p(\mat{x})dA$, so according to Gauss's law $ \oiint \mat{E}(\mat{\tilde{x}})\cdot d\mat{B}'=\int_{V'} \nabla\cdot \mat{E}(\mat{\tilde{x}}) dV'=\int_{V'}\tilde{p}(\tilde{\mat{x}})dV'=p(\mat{x})dA$, i.e.,
\begin{equation}
    \Phi'_{S_3} + \Phi'_{S_4} + \Phi'_{S_5} = p(\mat{x})dA,\quad \Phi'_{S_i} \equiv \oiint_{S_i} \mat{E}(\mat{\tilde{x}})\cdot d\mat{B}' \quad (i=3,4,5)
\end{equation}
The flux on the sides $\Phi'_{S_4}\propto\epsilon\to 0$, and $\Phi_{S_3}'=\Phi_{S_5}'$ due to mirror symmetry of $z=0$. So $\Phi_{S_3}'=\Phi_{S_5}'=p(\mat{x})dA/2$.  Note on the $S_3$ surface, the outflux of the pillbox is exactly the influx of the tube, so we have:
\begin{align*}
    \Phi_{S_3}=-\Phi_{S_3}' = -p(\x)dA/2\numberthis \label{eq:in},
\end{align*}
inserting which and $\Phi_{S_2}=0$ to Eq.~(\ref{eq:tube_zeronetflux}) gives 
\begin{equation}\label{eq:Phi_S1_1}
    \Phi_{S_1} = -\Phi_{S_3} = p(\mat{x})dA/2.
\end{equation}
On the other hand, in the far-field limit $r\to\infty$, since $\textrm{supp}(p(\rvx))$ is bounded, the data distribution can be effectively seen as a point charge (see Appendix~\ref{sec:mul}). By Lemma~\ref{lemma:point}, we have $\lim_{r\to\infty}\E(\r)=- \lim_{r\to\infty}\nabla\varphi(\r) = \frac{{\r}}{S_N(1) r^{N+1}}$. The resulting outflux on the hemisphere is
\begin{align*}
    \Phi_{S_1}=E_r r^N d\Omega=d\Omega/S_N(1) \numberthis \label{eq:out}
\end{align*}
where $E_r\equiv\mat{E}(\rvr) \cdot{\rvr}/{r}$ is the radial component of $\mat{E}$. Comparing \Eqref{eq:Phi_S1_1} and \Eqref{eq:out} yields 
$d\Omega/dA=p(\x)S_N(1)/2\propto p(\x)$.
In other words, the mapping from the $z=0$ hyperplane to the hemisphere dilutes the charge density ${p}({\mat{x}})$ up to a constant factor. Thus by change-of-varible, we conclude that the mapping transforms the data distribution into a uniform distribution on the infinite hemisphere. Since the ODE is reversible, the backward ODE transforms the uniform distributoin on the infinite hemisphere to the distribution $\tilde{p}(\tilde{\rvx})$.
\end{proof}

\subsection{Multipole Expansion}
\label{sec:mul}
We discuss the behaviors of the potential function in Poisson equation~(\Eqref{eq:poisson}) under different scenarios, utilizing the multipole expansion. Suppose we have a unit point charge $q=1$ located at $\mat{x}\in\mathbb{R}^N$. We know that the potential function at another point $\mat{y}\in\mathbb{R}^N$ is $\varphi(\mat{y}-\mat{x})=1/||\mat{y}-\mat{x}||^{N-2}$ (ignoring a constant factor). Now we assume that $\mat{x}$ is close to the origin such that we can Taylor expand around $\mat{x}=0$:
\begin{equation}
    \varphi(\mat{y}-\mat{x})= \varphi(\mat{y}) - \sum_{\alpha=1}^N \mat{x}_\alpha \varphi_\alpha(\mat{y})+\frac{1}{2}\sum_{\alpha=1}^N\sum_{\beta=1}^N \mat{x}_\alpha\mat{x}_\beta \varphi_{\alpha\beta}(\mat{y})-...
\end{equation}
where
\begin{equation}
\begin{aligned}
    &\varphi_\alpha(\mat{y})=\left(\frac{\partial \varphi(\mat{y}-\mat{x})}{\partial \mat{x}_\alpha}\right)_{\mat{x}=0}=(N-2)\frac{\mat{y}_\alpha}{||\mat{y}||^{N}} \\ &\varphi_{\alpha\beta}(\mat{y})=\left(\frac{\partial^2 \varphi(\mat{y}-\mat{x})}{\partial \mat{x}_\alpha\partial \mat{x}_\beta}\right)_{\mat{x}=0}=(N-2)\frac{N\mat{y}_\alpha\mat{y}_\beta-||\mat{y}||^2\delta_{\alpha\beta}}{||\mat{y}||^{N+2}}
\end{aligned}
\end{equation}
In the case where the source is a distribution $p(\mat{x})$, the potential $\varphi(\mat{y})$ can again be Taylor expanded:
\begin{equation}
        \varphi(\mat{y}) = q\varphi(\mat{y}) + \sum_{\alpha=1}^ Nq_\alpha\varphi_\alpha (\mat{y}) + \sum_{\alpha=1}^N\sum_{\beta=1}^N q_{\alpha\beta}\varphi_{\alpha\beta}(\mat{y}) -...
\end{equation}
where 
\begin{equation}
    q = \int p(\mat{x})d\mat{x}, q_\alpha = \int p(\mat{x})\mat{x}_\alpha d\mat{x}, q_{\alpha\beta} = \int p(\mat{x})\mat{x}_\alpha\mat{x}_\beta d\mat{x},
\end{equation}
which are called monopole, dipole and quadrupole in physics, respectively. The gradient field $\mat{E}\mat(y)=\nabla\Phi(\mat{y})$ can be expanded in the same such that
\begin{equation}
    \mat{E}(\mat{y}) = \mat{E}^{(0)}(\mat{y})+\mat{E}^{(1)}(\mat{y})+\mat{E}^{(2)}(\mat{y})+...
\end{equation}
It is easy to check that $||\mat{E}^{(i)}(\mat{y})||$ decays as $1/||\mat{y}||^{N-2+i}$, which means higher-order corrections decay faster than leading terms. So when $||\mat{y}||\to \infty$, only the monopole term $||\mat{E}^{(0)}(\mat{y})||$ matters, which behaves like a point source.

In a more realistic setup, we only have a large but finite $||\mat{y}||$, so the question is: under what condition is the point source approximation valid? We examine $\varphi^{(0)}$, $\varphi^{(1)}$ and $\varphi^{(2)}$ more carefully:
\begin{equation}\label{eq:phi_expansion}
\begin{aligned}
    &\varphi^{(0)} = \frac{1}{||\mat{y}||^{N-2}} \\
    &\varphi^{(1)} = \sum_{\alpha=1}^N (N-2)\frac{\mat{y}_\alpha\mat{x}_\alpha}{||\mat{y}||^N}=(N-2)\frac{\mat{x}^T\mat{y}}{||\mat{y}||^N}\\
    &\varphi^{(2)}=\frac{1}{2}\sum_{\alpha=1}^N\sum_{\beta=1}^N (N-2)\frac{N\mat{y}_\alpha\mat{y}_\beta-||\mat{y}||^2\delta_{\alpha\beta}}{||\mat{y}||^{N+1}}\mat{x}_\alpha\mat{x}_\beta=\frac{N-2}{2}\frac{N(\mat{x}^T\mat{y})^2-||\mat{x}||^2||\mat{y}||^2}{||\mat{y}||^{N+2}}
\end{aligned}
\end{equation}
Since $\varphi^{(1)}$ is an odd function of $\mat{x}$, integrating $\varphi^{(1)}$ over $\mat{x}$ leads to zero (samples are normalized to zero mean). In machine learning applications, $N$ is usually a large number (although in physics $N$ is merely 3). If $\mat{y}$ is a random vector of length $||\mat{y}||$, then $\mat{x}^T\mat{y}\sim (\frac{1}{\sqrt{N}}\pm\frac{1}{N})||\mat{x}||||\mat{y}||$. So Eq.~(\ref{eq:phi_expansion}) can be approximated as 
\begin{equation}
    \varphi^{(0)}\sim \frac{1}{||\mat{y}||^{N-2}}, \varphi^{(2)}\sim \frac{{\sqrt{N}}}{2}\frac{||\mat{x}||^2}{||\mat{y}||^N}
\end{equation}
Requiring $\int \varphi^{(0)} p(\rvx) d\rvx \gg \int \varphi^{(2)} p(\rvx) d \rvx$ gives $||\mat{y}||^2\gg \sqrt{N}\mathop{\mathbb{E}}_{p(x)}||\mat{x}||^2$. So the condition for the point source approximation to be valid is:
\begin{equation}
    \kappa=\frac{2||\mat{y}||^2}{\sqrt{N}\mathop{\mathbb{E}}_{p(x)}||\mat{x}||^2}\gg 1 
\end{equation}
Based on this condition, we can partition space into three zones: (1) the far zone $\kappa\gg1$, where the point source approximation is valid; (2) the intermediate zone $\kappa\sim O(1)$, where the gradient field has moderate curvature; (3) the near zone $\kappa\ll 1$, where the gradient field has high curvature. In practice, the initial value $||\rvy||$ is greater than 1000 (hence $\kappa \gg 1$) with high probability on CIFAR-10 and CelebA datasets, incidating that the initial samples lie in the far zone and gradually move toward the near zone where $||\rvy|| \approx ||\rvx||$ ($\kappa \ll 1$).

We summarize above observations in the following lemma in the $||\rvy|| \to \infty$ limit: 
\begin{lemma}
\label{lemma:point}
Assume the data distribution $p(\rvx) \in \gC^0$ has a compact support in $\mathbb{R}^N$, then the solution $\varphi$ to the Poisson equation $\nabla^2\varphi(\mat{x})=-p(\mat{x})$ with zero boundary condition at infinity satisfies $\lim_{\parallel \rvx \parallel_2 \to \infty} \nabla \varphi(\rvx) = -\frac{1}{S_{N-1}(1)}\frac{\rvx}{\parallel \rvx\parallel^N_2 }$.
\end{lemma}
\begin{proof}
By Lemma~\ref{lemma:ex-uni}, the gradient of the solution has the following form:
\begin{align*}
    \nabla \varphi(\mat{x}) = \int\ \nabla_\mat{x}G(\mat{x},\mat{y})p(\mat{y})d\mat{y},\quad \nabla_\mat{x} G(\mat{x},\mat{y}) = -\frac{1}{S_{N-1}(1)} \frac{\mat{x}-\mat{y}}{||\mat{x}-\mat{y}||^{N}}.
\end{align*}

Since $p(\rvx)$ has a bounded support, we assume $\max\{\parallel \rvx \parallel_2: p(\rvx) \not = 0\}<B$. On the other hand, we have
\begin{align*}
    \lim_{\parallel \rvx \parallel_2 \to \infty}\nabla_\mat{x} G(\mat{x},\mat{y}) =\lim_{\parallel \rvx \parallel_2 \to \infty} -\frac{1}{S_{N-1}(1)} \frac{\mat{x}-\mat{y}}{||\mat{x}-\mat{y}||^{N}}=\lim_{\parallel \rvx \parallel_2 \to \infty}-\frac{1}{S_{N-1}(1)} \frac{\mat{x}}{||\mat{x}||^{N}}
\end{align*}
for $\forall y$ such that $\parallel y \parallel_2 < B$. Hence,
\begin{align*}
    \lim_{\parallel \rvx \parallel_2 \to \infty} \nabla \varphi(\rvx) =\lim_{\parallel \rvx \parallel_2 \to \infty}\int\ \nabla_\mat{x}G(\mat{x},\mat{y})p(\mat{y})d\mat{y}&= \int\lim_{\parallel \rvx \parallel_2 \to \infty} \nabla_\mat{x}G(\mat{x},\mat{y})p(\mat{y})d\mat{y}\\
    &=-\frac{1}{S_{N-1}(1)}\frac{\rvx}{\parallel \rvx\parallel^N_2 }
\end{align*}

\end{proof}



\subsection{{Extension of Green's Function in $N$-dimensional Space}}
\label{app:green}
In this section, we show that the function $G(\rvx, \rvy)$ defined in \Eqref{eq:poisson_solution} is the $N$-dimensional extension of the Green's function,     $\varphi(\mat{x}) = \int G(\mat{x},\mat{y})\rho(\mat{y})d\mat{y}$ solves the Poisson equation $\nabla^2\varphi(\mat{x}) = -\rho(\mat{x})$.
\begin{lemma}
\label{lemma:green}
Assume the dimension $N\ge 3$, and the source term satisfies $\rho \in \gC^0(\Omega), \int_{\mathbb{R}^N}\rho^2(\rvx)d\rvx<+\infty, \lim_{\parallel \rvx \parallel_2 \to \infty} \rho(\rvx) = 0$. The extension of Green's function $G(\mat{x},\mat{y})=\frac{1}{(N-2)S_{N-1}(1)}\frac{1}{||\mat{x}-\mat{y}||^{N-2}}$ solves the Poisson equation $\nabla^2_{\mat{x}} G(\mat{x},\mat{y})=-\delta(\mat{x}-\mat{y})$. In addition, with zero boundary condition at infinity~($\lim_{\parallel \rvx \parallel_2 \to \infty}\varphi(\mat{x})=0$), $\varphi(\mat{x}) = \int G(\mat{x},\mat{y})\rho(\mat{y})d\mat{y}$ solves the Poisson equation $\nabla^2\varphi(\mat{x}) = -\rho(\mat{x})$.
\end{lemma}

\begin{proof}
It is convenient to denote $\mat{r}=\mat{x}-\mat{y}$, $r=||\mat{r}||$ and notice $\partial r/\partial \mat{x}=\mat{r}/r$. Firstly, we calculate $\nabla_{\mat{x}} G(\mat{x},\mat{y})$:
\begin{equation}
\begin{aligned}
    \nabla_{\mat{x}} G(\mat{x}, \mat{y}) & = \frac{1}{(N-2)S_{N-1}(1)}\nabla_{\mat{x}}(\frac{1}{r^{N-2}}) \\
    & = \frac{1}{(N-2)S_{N-1}(1)} \frac{\partial}{\partial r}(\frac{1}{r^{N-2}})\nabla_\mat{x}r \\
    & = -\frac{1}{S_{N-1}(1)}\frac{\mat{r}}{r^N}
\end{aligned}
\end{equation}
Then we calculate $\nabla_{\mat{x}}^2 G(\mat{x},\mat{y})$:
\begin{equation}
    \begin{aligned}
        \nabla_{\mat{x}}^2 G(\mat{x},\mat{y}) &\equiv \nabla_{\mat{x}}\cdot\nabla_{\mat{x}}G(\mat{x},\mat{y}) \\
        & = -\frac{1}{S_{N-1}(1)}\nabla_{\mat{x}}\cdot \frac{\mat{r}}{r^N} \\ 
        & = -\frac{1}{S_{N-1}(1)}(\nabla_{\mat{x}}(\frac{1}{r^N})\cdot\mat{r}+\frac{1}{r^N}\nabla_{\mat{r}}\cdot\mat{r}) \\
        & = -\frac{1}{S_{N-1}(1)}(-\frac{N}{r^N}+\frac{N}{r^N}) \\
        & = -\frac{0}{S_{N-1}(1)r^N}
    \end{aligned}
\end{equation}
which is 0 for $r>0$, but undermined for $r=0$. So we are left with proving
\begin{equation}
    \int_{S_\epsilon(\mat{y})}\nabla_{\mat{x}}^2 G(\mat{x},\mat{y}) d^N\mat{x} = -1,
\end{equation}
where $S_\epsilon(\mat{y})$ denotes a ball centered at $\mat{y}$ with a radius $\epsilon\to 0^+$. With the divergence theorem, we have
\begin{equation}
    \int_{S_\epsilon(\mat{y})}\nabla_{\mat{x}}^2 G(\mat{x},\mat{y}) d^N\mat{x} =  \oiint_{\partial S_\epsilon(\mat{y})} \nabla_{\mat{x}}G(\mat{x},\mat{y})\cdot d^{N-1}\mat{B}
\end{equation}
where the surface integral can be computed
\begin{equation}
    \oiint_{\partial S_\epsilon(\mat{y})} \nabla_{\mat{x}}G(\mat{x},\mat{y})\cdot d^{N-1}\mat{B}  = \oiint_{\partial S_\epsilon(\mat{y})} (-\frac{1}{S_{N-1}(1)}\frac{\mat{r}}{r^N})\cdot d^{N-1}\mat{B} = -\frac{1}{S_{N-1}(1)}\frac{S_{N-1}(\epsilon)}{\epsilon^{N-1}} = -1
\end{equation}
in which we used $\oiint_{\partial S_\epsilon(\mat{y})}\mat{r}\cdot d^{N-1}\mat{B}=\epsilon S_{N-1}(\epsilon)$. Together, we conclude that
\begin{align*}
    \nabla^2_{\mat{x}} G(\mat{x},\mat{y})=-\delta(\mat{x}-\mat{y}) \numberthis \label{eq:G-delta}
\end{align*}
Next we show that $\varphi(\mat{x}) = \int G(\mat{x},\mat{y})\rho(\mat{y})d\mat{y}$ solves $\nabla^2 \varphi(\mat{x}) = -\rho(\rvx)$. Taking the Laplacian operator of both sides gives:
\begin{align*}
    \nabla^2_{\mat{x}} \varphi(\mat{x}) &= \nabla^2_{\mat{x}} \int G(\mat{x},\mat{y})\rho(\mat{y})d\mat{y}\\
    &=\int \nabla^2_{\mat{x}} G(\mat{x},\mat{y})\rho(\mat{y})d\mat{y}\\
    &= \int  -\delta(\mat{x}-\mat{y})\rho(\mat{y})d\mat{y}\quad \textrm{(By \Eqref{eq:G-delta})}\\
    &= -\rho(\mat{x})
\end{align*}

In addition, we show that $\varphi(\rvx)$ is zero at infinity. Since $\rho(\rvx) \in \gC^0$ and has compact support, we know that $\rho(\rvx)$ is bounded, and let $|\rho(\rvx)|<B$.
\begin{align*}
    \lim_{\parallel \rvx \parallel_2 \to \infty}\varphi(\mat{x}) &= \lim_{\parallel \rvx \parallel_2 \to \infty}\int G(\mat{x},\mat{y})\rho(\mat{y})d\mat{y}\\
    &\le B\lim_{\parallel \rvx \parallel_2 \to \infty}\int_{\textrm{supp}(\rho)}\frac{1}{(N-2)S_{N-1}(1)}\frac{1}{||\mat{x}-\mat{y}||^{N-2}}d\mat{y}\\
    &=0
\end{align*}
The last equality holds since $\textrm{supp}(\rho)$ is a compact set.
\end{proof}

\subsection{Proof for the Prior Distribution on $z=\zmax$ Hyperplane}\label{app:prior_distribution}

\begin{figure}[htbp]
    \centering
    \includegraphics[width=0.5\linewidth]{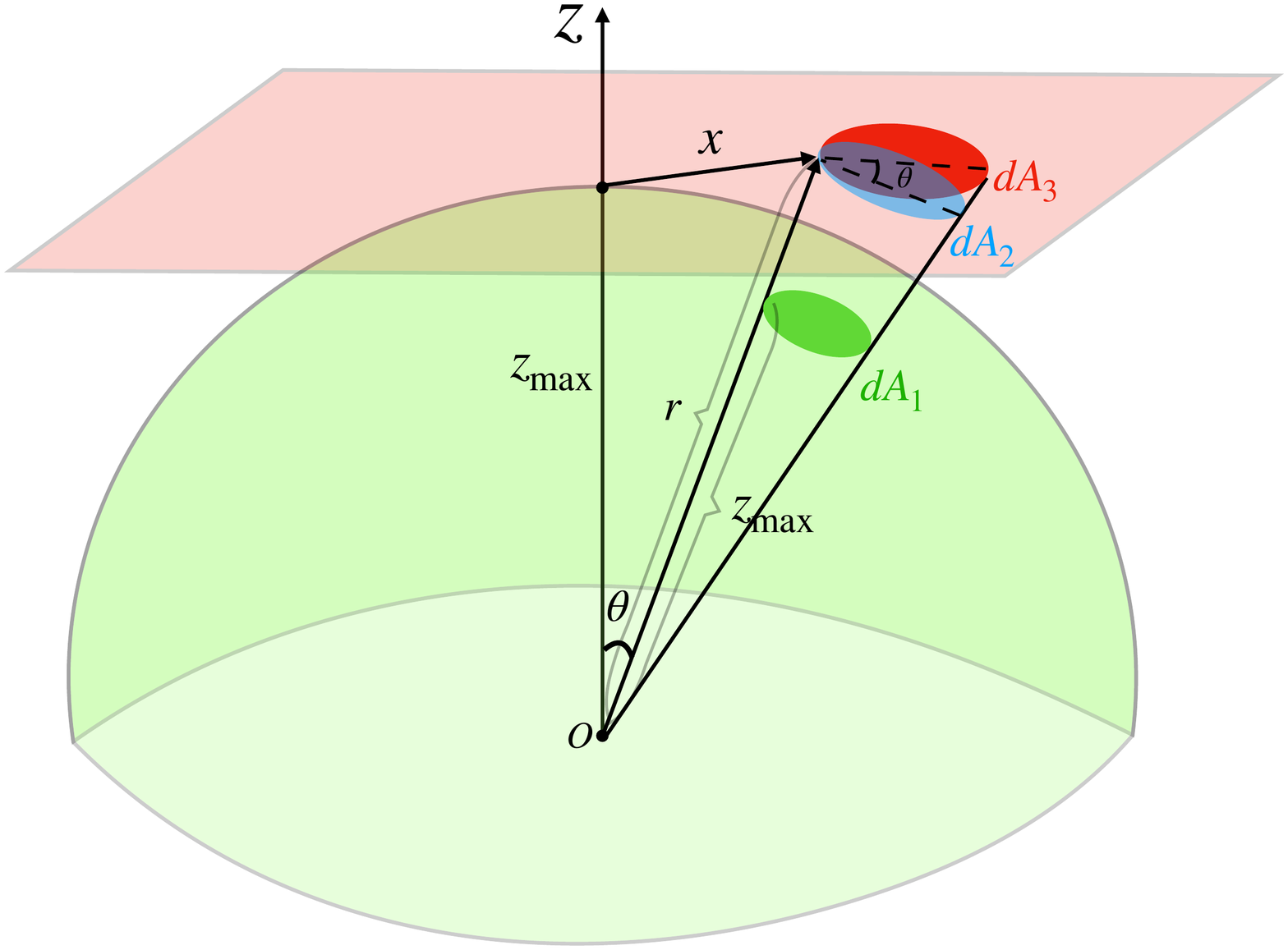}
    \caption{Diagram of the deviation in Proposition~\ref{prop:prior}}
    \label{fig:init_dist}
\end{figure}

We obtain the prior distribution $p_{\textrm{prior}}$ by projecting the uniform distribution $\gU(S_N^+(\zmax))$ on the hemisphere $S_N^+(\zmax)$ to the $z=\zmax$ hyperplane. In the following proposition, we show that the projected distribution is $p_{\textrm{prior}}(\mat{x})=\frac{2\zmax}{S_N(1)r^{N+1}}$. 
\begin{proposition}
\label{prop:prior}
The radial projection of $\gU(S_N^+(\zmax))$ on the hemisphere $S_N^+(\zmax)$ to the $z=\zmax$ hyperplane is $p_{\textrm{prior}}(\mat{x})=\frac{2\zmax}{S_N(1)r^{N+1}}$.
\end{proposition}
\begin{proof}
We calculate the change-of-variable ratio by comparing two associate areas. As illustrated in \Figref{fig:init_dist}, an area $dA_1$ on $S_N^+(\zmax)$ is projected to an area $dA_3$ on the hyperplane in the $(\rvx,\zmax)$ direction, and we have $$\gU(S_N^+(\zmax))dA_1=p_{\textrm{prior}}(\mat{x})dA_3$$ We aim to calculate the ratio $dA_1/dA_3$ below. We define the angle between $(\mat{0}, \zmax)$ and $\tilde{\mat{x}}=(\mat{x},\zmax)$ to be $\theta$. We project $dA_3$ to the hyperplane orthogonal to $\tilde{\mat{x}}$ to get $dA_2=dA_3{\rm cos}\theta=dA_3\zmax/r$ where $r\equiv ||\mat{\tilde{x}}||_2=\sqrt{||\mat{x}||_2^2+\zmax^2}$. Since $dA_1$ is parallel to $dA_2$ and they lie in the same cone from the origin $O$, we have $dA_2/dA_1=(r/\zmax)^N$. Combining all the results gives
\begin{align*}
    p_{\textrm{prior}}(\mat{x})=\gU(S_N^+(\zmax))\frac{dA_1}{dA_3} = \gU(S_N^+(\zmax))\frac{dA_1}{dA_2}\frac{dA_2}{dA_3}=\frac{2}{S_N(1)\zmax^N}(\frac{\zmax}{r})^N\frac{\zmax}{r}=\frac{2\zmax}{S_N(1)r^{N+1}}.
\end{align*}
\end{proof}

In order to sample from $p_{\textrm{prior}}(\rvx)$, we first sample the norm~(radius) $R=||\rvx||_2$ from the distribution:
\begin{align*}
        p_{\textrm{radius}}(R)&\propto R^{N-1}p_{\textrm{prior}}(\mat{x})\qquad \textrm{($p_{\textrm{prior}}$ is isotropic)}\\
        &\propto R^{N-1}/{(||\rvx||_2^2+\zmax^2)^{\frac{N+1}{2}}}\\
        &= R^{N-1}/{(R^2+\zmax^2)^{\frac{N+1}{2}}} \numberthis \label{eq:propto}
\end{align*}
and then uniformly sample its angle. Sampling from $p_{\textrm{prior}}$ encompasses three steps. We first sample a real number $r_1$ with parameters $\alpha=\frac{N}{2}, \beta=\frac12$, \ie
\begin{align*}
    R_1 \sim \textrm{Beta}(\alpha, \beta)
\end{align*}
Next, we set $R_2=\frac{R_1}{1-R_1}$ such that $R_2$ is effectively sampled from the inverse beta distribution~a(also known as beta prime distribution) with parameters $\alpha=\frac{N}{2}, \beta=\frac12$. Finally, we set $R_3 = \sqrt{\zmax^2 R_2}$. To verify the pdf of $R_3$ is $p_{\textrm{radius}}$, note that the pdf of inverse beta distribution is
\begin{align*}
    p(R_2) \propto R_2^{\frac{N}{2}-1}(1+R_2)^{-\frac{N}{2}-\frac12}
\end{align*}
Next, by change-of-variable, the pdf of $R_3=\sqrt{\zmax^2 R_2}$ is
\begin{align*}
    p(R_3) &\propto R_2^{\frac{N}{2}-1}(1+R_2)^{-\frac{N}{2}-\frac12}*\frac{2R_3}{\zmax^2}\\
    &\propto\frac{R_3R_2^{\frac{N}{2}-1}}{(1+R_2)^{\frac{N+1}{2}}}\\
    &= \frac{(R_3/\zmax)^{N-1}}{(1+(R_3^2/\zmax^2))^{\frac{N+1}{2}}}\\
    &\propto  \frac{R_3^{N-1}}{(1+(R_3^2/\zmax^2))^{\frac{N+1}{2}}}\\
    &\propto  \frac{R_3^{N-1}}{(\zmax^2+R_3^2)^{\frac{N+1}{2}}} \propto p_{\textrm{radius}}(R_3)\qquad \textrm{(By \Eqref{eq:propto})}
\end{align*}

Hence we conclude that $p(R_3) = p_{\textrm{radius}}(R_3)$.

\section{Experimental Details}

\subsection{Training}
\label{app:training}

In this section we include more details about the training of PFGM and other baselines. We show the hyper-parameters settings for all the baselines~(Appendix~\ref{app:hyper-train}). All the experiments are run on a single NVIDIA A100 GPU.

\subsubsection{Additional Settings}
\label{app:hyper-train}

\paragraph{PFGM} We set the hyper-parameters $\gamma=5$, the larger batch size for calculating normalize field $|\gB_L|=2048~\textrm{(CIFAR-10)}, 256~\textrm{(CelebA)}, 64~\textrm{(LSUN bedroom)}$ in Algorithm~\ref{alg:pf}, and $M=291~\textrm{(CIFAR-10, CelebA)}/356~\textrm{ (LSUN bedroom)}$, $\sigma=0.01$ and $\tau=0.03$ in Algorithm~\ref{alg:ode}. We use the a batch size of $|\gB|=128~\textrm{(CIFAR-10, CelebA)}/32~\textrm{ (LSUN bedroom)}$, the same Adam optimizer and exponential moving average in \cite{Song2021ScoreBasedGM}. We center the data around the origin. The initial $z$ components in the normalized field are approximately zero with small initial $|\epsilon_z|$ values in Algorithm~\ref{alg:ode}. In this case, the trajectories of the forward ODE terminate at points that are unlikely traversed by the backward ODE, \ie points with large $\parallel \rvx\parallel_2$ and small $z$. In light of this, we heuristically confine the maximum sampling step to $M=200~\textrm{ (CIFAR-10, CelebA)}/250~\textrm{ (LSUN bedroom)}$ for points with the initial $|\epsilon_z|$ smaller than $0.005$. More principal solutions are left for future works.

For \textit{selecting $M$ in more general settings}, we recommend the following rule-of-thumb. According to analysis in \Secref{sec:mul}, given a perturbation point $(\rvy,z)$ when setting the exponent $m=M$ in Algorithm~\ref{alg:ode}, we can ensure the point source approximation by
\begin{align*}
    {||\mat{y}||^2}\gg {\sqrt{N}{\mathbb{E}}_{p(\rvx)}||\mat{x}||^2}/2 \numberthis \label{eq:rule-M}
\end{align*}
where $N$ is the data dimension and $p(\rvx)$ is the data distribution. By WLLN, we have $||\epsilon_\rvx|| = \sqrt{N}\sigma$, and recall that $
    \rvy = \rvx+\parallel \epsilon_\rvx \parallel(1+\tau)^M\rvu
$
where $\epsilon = (\epsilon_\rvx, \epsilon_z) \sim \gN(0, \sigma^2I_{N+1\times N+1})$, $\rvu\sim \gU(S_N(1))$. Together, we conclude $||\rvy|| \approx \sqrt{N}\sigma(1+\tau)^M$. Substituting in \Eqref{eq:rule-M}, we have 
\begin{align*}
    M > \frac12 \log_{1+\tau}{\frac{\mathbb{E}_{p(\rvx)}||\rvx||^2}{2\sqrt{N}\sigma^2}} = \frac12 \frac{\ln{\frac{\mathbb{E}_{p(\rvx)}||\rvx||^2}{2\sqrt{N}\sigma^2}}}{\ln{1+\tau}}
\end{align*}
We empirically observe that setting $M=  \frac34 \frac{\ln{\frac{\mathbb{E}_{p(\rvx)}||\rvx||^2}{2\sqrt{N}\sigma^2}}}{\ln{1+\tau}}$ already gives good results, and the corresponding $||\rvy||\approx 3000$. For example, on CIFAR-10 datasets, $N=3072, \tau=0.03, \sigma=0.01, \mathbb{E}_{p(\rvx)}||\rvx||^2\approx 900$, we have $M=\frac34 \frac{\ln{\frac{\mathbb{E}_{p(\rvx)}||\rvx||^2}{2\sqrt{N}\sigma^2}}}{\ln{1+\tau}} \approx 291$. 

Since we are operating in the augmented space, we add minor modifications to the DDPM++/DDPM++ deep architectures to accommodate the extra dimension. More specifically, we replace the conditioning time variable in VP/sub-VP with the additional dimension $z$ in PFGM as the input to the positional embedding. We also need to add an extra scalar output representing the $z$ direction. To this end, we add an additional output channel to the final convolution layer and take the global average pooling of this channel to obtain the scalar. {For LSUN bedroom dataset, we both experiments with the channel configurations suggested in NSCN++~\cite{Song2021ScoreBasedGM} and DDPM~\cite{Ho2020DenoisingDP}.}

\paragraph{VE/VP/sub-VP} We use the same set of hyper-parameters and the NCSN++/DDPM++ (deep) backbone and the continuous-time training objectives for forward SDEs in \cite{Song2021ScoreBasedGM}.

\subsection{Sampling}
\label{app:sampling}

We provide more details of PFGM and VE/VP sampling implementations in Appendix~\ref{app:sample-add}. We further discuss two techniques used in PFGM ODE sampler: change-of-variable formula~(Appendix~\ref{app:exp}) and the substitution of ground-truth Poisson field direction on $z$~(Appendix~\ref{app:sub}).

\subsubsection{Additional settings}
\label{app:sample-add}
\paragraph{PFGM} For RK-45 sampler, we use the function implemented in \texttt{scipy.integrate.solve\_ivp} with \texttt{atol}=$1e-4$, \texttt{rtol}=$1e-4$. For forward Euler method, we discretize the ODE with constant step size determined by the number of steps, \ie step size = $(\log \zmax - \log \zmin)$/number of steps for the backward ODE~(\Eqref{eq:backode}). {As in [1], we set the terminal value of $z_{min}=1e-3$. We choose $\zmax = 40 \textrm{ (CIFAR-10)},60 \textrm{ (CelebA $64^2$)}, 100\textrm{ (LSUN bedroom)}$ to satisfy the condition $\kappa \gg 1$ by the multipole expansion analysis in Appendix~\ref{sec:mul}. The condition ensures that the data distribution can be viewed roughly as a point source at origin. For example, we set $z_{max}=40$ on CIFAR-10, and the corresponding $\kappa$ is greater than $50$ with high probability. The hyperparameters work well without further fine tuning. Hence, we hypothesize that PFGM is insensitive to the choice of hyperparameters in a reasonable range, as shown in Table~\ref{table:fid-zmax}.} We clip the norms of initial samples into $(0, 3000)$ for CIFAR-10, $(0,6000)$ for CelebA and $(0, 30000)$ for LSUN bedroom.

For selecting $\zmax$ and clipping upper bound of norms for general datasets, we recommend the following rule-of-thumb. Recall that during the training perturbations~(\Eqref{eq:geo-ode}), given a random initial value $\epsilon_z \sim \gN(0, \sigma^2)$, maximum $z$ is
\begin{align*}
     z=|\epsilon_z| (1+\tau)^M
\end{align*}
Hence we set $\zmax = \mathbb{E}[|\epsilon_z| (1+\tau)^M]=\sqrt{\frac{2}{\pi}}\sigma(1+\tau)^M$. For example, on CIFAR-10, $\tau=0.03, M=291$, and $\zmax\approx 43$. The clipping upper value is similarity derived, by setting it to $\mathbb{E}[||\epsilon_\rvx|| (1+\tau)^M] = \sqrt{N}\sigma (1+\tau)^M\approx 3000$, where $\epsilon_\rvx\sim \gN(0, \sigma^2I_{N\times N})$. By combining \Eqref{eq:rule-M}, we further have
\begin{align*}
    &\zmax = \sqrt{\frac{2}{\pi}}\sigma(1+\tau)^M = \sqrt{\frac{2}{\sigma\pi}}\left(\frac{\mathbb{E}_{p(\rvx)}||\mat{x}||^2}{2\sqrt{N}}\right)^{\frac34}\\
    &\textrm{clipping upper value}=\sqrt{N}\sigma (1+\tau)^M= \sqrt{\frac{N}{\sigma}}\left(\frac{\mathbb{E}_{p(\rvx)}||\mat{x}||^2}{2\sqrt{N}}\right)^{\frac34}
\end{align*}
where $N$ is the data dimension and $p(\rvx)$ is the data distribution. These formulas are easier for practitioner to apply PFGM on new datasets.

\paragraph{VE/VP/sub-VP} For the PC sampler in VE, we follow \cite{Song2021ScoreBasedGM} to set the reverse diffusion process as the predictor and the Langevin dynamics (MCMC) as the corrector. For VP/sub-VP, we drop the corrector in PC sampler since it only gives slightly better results~\cite{Song2021ScoreBasedGM}.
\begin{table*}[htb]
\begin{center}
\caption{FID scores versus $z_{max}$ on PFGM w/ DDPM++}
\label{table:fid-zmax}
\begin{tabular}{c c c c c c c c}
		\toprule
		\textbf{$\bm{z_{max}}$} &  $30$ &$40$ & $50$\\
		\midrule
        \textbf{FID score} & {2.49} & {2.48} &{2.48}\\
        \bottomrule
\end{tabular}
\end{center}
\end{table*}
\subsubsection{Exponential Decay on $z$ Dimension}
\label{app:exp}
Recall that in Section~\ref{sec:sampling}, we replace the vanilla backward ODE with a new ODE anchored by $z$:
\begin{align*}
    d(\rvx,z) = (\frac{d \rvx}{dt}\frac{d t}{dz}dz,dz) = (\rvv(\tilde{\rvx})_\rvx\rvv(\tilde{\rvx})_z^{-1}, 1) dz
\end{align*}
We further use the change-of-variable formula, \ie $t'=-\log z$, to achieve exponential decay on the $z$ dimension:
\begin{align*}
{d(\rvx,z)} &= (\rvv(\tilde{\rvx})_\rvx\rvv(\tilde{\rvx})_z^{-1}z, z){dt'} 
\end{align*}

The trajectories of the two ODEs above are the same when $dt, dt' \to 0$. We compare the NFE and the sample quality of different ODEs in Table~\ref{table:exp}. We measure the NFE/FID of generating 50000 CIFAR-10 samples with the RK45 method in Scipy package~\cite{Virtanen2020SciPy1F}. The batch size is set to $1000$. All the numbers are produced on a single NVIDIA A100 GPU. We observe that the ODE with the anchor variable $t'$ not only accelerates the vanilla by 2 times, but has almost no harm to the sample quality measured by FID score.
\begin{table*}[htb]
\begin{center}
\caption{NFE and FID scores of different backward ODEs in PFGM}
\label{table:exp}
\begin{tabular}{c c c c c c c c}
		\toprule
		\textbf{Algorithm} &  $d(\rvx,z)/dz$ &$d(\rvx,z)/dt'$\\
		\midrule
        \textbf{NFE} &  $242$ &$104$ \\
        \textbf{FID score} & 2.53 & 2.48 \\
        \bottomrule
\end{tabular}
\end{center}
\end{table*}

\subsubsection{Substitute the Predicted $z$ Direction with the Ground-truth}
\label{app:sub}
Since the neural network cannot perfectly learn the ground-truth $z$ direction, we replace the predicted $f_\theta(x)_z$ with the ground-truth direction when $z$ is small. More specifically, given $\tilde{\rvx} = (\rvx, z) \in \mathbb{R}^{N+1}$, recall that the empirical field is $    \hat{\mat{E}}(\tilde{\mat{x}}) = c(\tilde{\mat{x}})\sum_{i=1}^n \frac{\tilde{\mat{x}}-\tilde{\mat{x}}_i}{||\tilde{\mat{x}}-\tilde{\mat{x}}_i||^{N+1}}$ where $c(\tilde{\mat{x}})=1/\sum_{i=1}^n \frac{1}{||\tilde{\mat{x}}-\tilde{\mat{x}}_i||^{N+1}} $. Hence we can rewrite the empirical field as 
\begin{align*}
    \hat{\mat{E}}(\tilde{\mat{x}}) =\sum_{i=1}^n w(\tilde{\mat{x}},\tilde{\mat{x}}_i) ({\tilde{\mat{x}}-\tilde{\mat{x}}_i})
\end{align*}
where $\sum_{i=1}^n w(\tilde{\mat{x}},\tilde{\mat{x}}_i) = \sum_{i=1}^n \frac{\frac{1}{||\tilde{\mat{x}}-\tilde{\mat{x}}_i||^{N+1}}}{\sum_{j=1}^n\frac{1}{||\tilde{\mat{x}}-\tilde{\mat{x}}_j||^{N+1}}}=1$. Furthermore we have $\forall i, ({\tilde{\mat{x}}-\tilde{\mat{x}}_i})_z = z - 0 = z$. Together, the $z$ component in the empirical field is $\hat{\mat{E}}(\tilde{\mat{x}})_z =\sum_{i=1}^n w(\tilde{\mat{x}},\tilde{\mat{x}}_i) ({\tilde{\mat{x}}-\tilde{\mat{x}}_i})_z = z$.
The predicted normalized field~(on $\rvx$) is trained to approximate the normalized field~(on $\rvx$), \ie 
\begin{align*}
f_\theta(\tilde{\rvx})_\rvx &\approx -\sqrt{N}\hat{\mat{E}}(\tilde{\rvx})_\rvx/(\sqrt{\parallel \hat{\mat{E}}(\tilde{\rvx})_\rvx \parallel_2^2+ z^2}+\gamma)\\&\approx -\sqrt{N}\hat{\mat{E}}(\tilde{\rvx})_\rvx/(\sqrt{\parallel \hat{\mat{E}}(\tilde{\rvx})_\rvx \parallel_2^2}+\gamma)    
\end{align*}
The last approximation is due to $\parallel \hat{\mat{E}}(\tilde{\rvx})_\rvx \parallel_2 \gg z$. Solving for $\parallel \hat{\mat{E}}(\tilde{\rvx})_\rvx\parallel_2$, we get $
    \parallel \hat{\mat{E}}(\tilde{\rvx})_\rvx\parallel_2 \approx \frac{\gamma \parallel f_\theta(\tilde{\rvx})_\rvx \parallel_2/\sqrt{N}}{1 - \parallel f_\theta(\tilde{\rvx})_\rvx \parallel_2/\sqrt{N}}
$. Hence the $z$ component in the normalized field after substituting the ground-truth is $\hat{\mat{E}}(\tilde{\rvx})_z/(\sqrt{\parallel \hat{\mat{E}}(\tilde{\rvx})_\rvx\parallel_2^2+z^2}+\gamma)=z/(\sqrt{(\frac{\gamma \parallel f_\theta(\tilde{\rvx})_\rvx \parallel_2/\sqrt{N}}{1 - \parallel f_\theta(\tilde{\rvx})_\rvx \parallel_2/\sqrt{N}})^2+z^2} + \gamma)$. In our experiments, we therefore replace the original prediction $f_\theta(\tilde{\rvx})_z$ with $-\sqrt{N}z/(\sqrt{(\frac{\gamma \parallel f_\theta(\tilde{\rvx})_\rvx \parallel_2/\sqrt{N}}{1 - \parallel f_\theta(\tilde{\rvx})_\rvx \parallel_2/\sqrt{N}})^2+z^2} + \gamma)$ when $z< 5/1/0.1$ during the backward ODE sampling for CIFAR-10/CelebA $64^2$/LSUN bedroom $256^2$. 

Table~\ref{table:subz} reports the NFE and FID score w/o and w/ the above substitution. We observe that the usage of ground-truth $z$ direction in the near field accelerates the sampling speed.

\begin{table*}[htb]
\begin{center}
\caption{NFE and FID scores of w/ and w/o substitution}
\label{table:subz}
\begin{tabular}{c c c c}
		\toprule
		\textbf{Algorithm} &  w/o substitution &w/ substitution\\
		\midrule
        \textbf{NFE} &  $134$ &$104$ \\
        \textbf{FID score} & $2.48$ & $2.48$ \\
        \bottomrule
\end{tabular}
\end{center}
\end{table*}

\subsection{Evaluation}

We use FID~\cite{Heusel2017GANsTB} and Inception scores~\cite{Salimans2016ImprovedTF} to quantitatively measure the sample quality, and NFE~(number of evaluation steps) for the inference speed. {FID (Fréchet Inception Distance) score is the Fréchet distance between two multivariate Gaussians, whose means and covariances are estimated from the 2048-dimensional activations of the Inception-v3~\citep{Szegedy2016RethinkingTI} network for real and generated samples respectively.} Inception score is the exponential mutual information between the predicted labels of the Inception network and the images. We also report bits/dim for likelihood evaluation. It is computed by dividing the negative log-likelihood by the data dimension, \ie $\textrm{bits/dim} = -\log p_{\textrm{prior}}(\rvx)/N$.

For CIFAR-10, we compute the Fréchet distance between 50000 samples and the pre-computed statistics of CIFAR-10 dataset in \cite{Heusel2017GANsTB}. For CelebA $64 \times 64$, we follow the setting in \cite{Song2020ImprovedTF} where the distance is computed between 10000 samples and the test set. For model selection, we follow \cite{Song2020ImprovedTF} and pick the checkpoint with smallest FID every 50k iterations on 10k samples for computing all the scores.

\subsection{Effects of Step Size: FID versus NFE}

For preciseness, Table~\ref{table:fig5c} reports the exact numbers in \Figref{fig:adapt}. 
\begin{table*}[htb]
\begin{center}
\caption{The FID scores in \Figref{fig:adapt} of different methods and NFE.}
\label{table:fig5c}
\begin{tabular}{c c c c c}
		\toprule
		\textbf{Method / NFE} &  10 & 20 & 50 & 100\\
		\midrule
			\textbf{VP-ODE} & $192.36$& $72.25$ &$38.18$& $19.73$\\
			\textbf{DDIM} &$13.36$& $6.48$& $4.67$& $4.16$\\
        	\textbf{PFGM} &  $14.98$ & $6.46$& $3.48$ & $2.89$ \\
        \bottomrule
\end{tabular}
\end{center}
\end{table*}

Since in the ODE $ {d(\rvx,z)} = -(\rvv(\tilde{\rvx})_\rvx\rvv(\tilde{\rvx})_z^{-1}z, z){dt'}$ of PFGM, the $z$ variable is a function of $t'$~($z=e^{t'}$), we integrate the $z$ in the Euler method to reduce the discretization error. The vanilla update from time $t'_i$ to time $t'_{i+1}$ is ${(\rvx_{i+1},z_{i+1})} = (\rvx_{i},z_{i})-(\rvv(\tilde{\rvx}_i)_\rvx\rvv(\tilde{\rvx}_i)_{z_i}^{-1}z_i, z_i)(t'_{i+1}-t'_i)$, and the new update is ${(\rvx_{i+1},z_{i+1})} = (\rvx_{i},z_{i})-(\rvv(\tilde{\rvx}_i)_\rvx\rvv(\tilde{\rvx}_i)_{z_i}^{-1}\int_{t'_i}^{t'_{i+1}}z(t')dt', \int_{t'_i}^{t'_{i+1}} z(t')dt')$. We empirically observe that the new update scheme significantly improve the FID score.

\section{Failure of VE/VP-ODE on NCSNv2 backbone}
\label{app:failure}

In \Figref{fig:ve}, we demonstrate the trajectories of cleaner samples/noisier samples/noisier samples w/ corrector. We visualize these three groups in \Figref{fig:failure-a} and \Figref{fig:failure-b}. The noisier samples are marked with red boxes in \Figref{fig:failure-a} and the remaining images in \Figref{fig:failure-a} are cleaner samples. The samples within green boxes in \Figref{fig:failure-b} are noisier samples w/ corrector. Samples on the same spatial locations in the two figures are generated by identical initial latents. 

The Gaussian kernels in score-based models are $\gN(\rvx, \sigma(t)^2)$~(VE) and $\gN(\sqrt{1-\sigma(t)^2}\rvx,\sigma(t)^2)$~(VP) \cite{Song2021ScoreBasedGM}. When $\sigma(t)$ is large, the norms of perturbed samples are approximately $\sqrt{N}\sigma(t)$. The backward ODE could break down if the trajectories diverge from the norm-$\sigma(t)$ relation, as shown by the noisier samples' trajectories in \Figref{fig:ve}. In contrast, the norm distributions of PFGM is approximately $p(\parallel\rvx\parallel)\propto {\parallel \rvx \parallel_2^{N-2}}/{(\parallel \rvx \parallel_2^2+z^2)^{\frac{N}{2}}}$ when $z$ is large~(see deviation for $p_{\textrm{prior}}$ in Appendix~\ref{app:prior_distribution}), which have a wider span for high density region~(see \Figref{fig:compare_prior}). The weak correlation between norm and $z$ makes PFGM more robust on the lighter NCSNv2 backbone.

\begin{figure*}
    \centering
    \subfigure[Samples from VE-ODE (Euler)]{    \label{fig:failure-a}\includegraphics[width=0.4\textwidth]{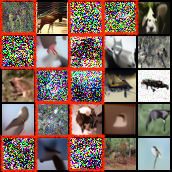}}\hspace{10pt}
    \subfigure[Samples from VE-ODE (Euler w/ corrector)]{\label{fig:failure-b}\includegraphics[width=0.4\textwidth]{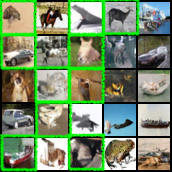}}
    \caption{\textbf{(a)} Samples from VE-ODE (Euler w/o corrector). We highlight the noisier images with red boxes. The rest are cleaner images. \textbf{(b)} Samples from VE-ODE (Euler w/ corrector). We mark the noisier samples after correction with green boxes.}
\end{figure*}

\section{Extra Experiments}
\label{app:extra-exp}
\subsection{LSUN Bedroom $256\times 256$}
\label{app:exp-lsun}
We report the FID scores and NFEs for LSUN bedroom dataset in Table~\ref{tab:lsun}. We adopt the code base of \cite{Song2021ScoreBasedGM} in our experiments. In \cite{Song2021ScoreBasedGM}, they experimented on the LSUN bedroom 256$\times$ 256 dataset only on VE-SDE using a deeper NCSN++ backbone. In our DDPM++ architecture, we directly borrow the configuration of channels from the NCSN++ architecture~\cite{Song2021ScoreBasedGM} in each residual block (PFGM w/ NCSN++ channel). We further change $z_{max}$ to $100$, as it empirically gives better sample quality.

We also evaluate the performance when using the configuration of channels in the DDPM~\cite{Ho2020DenoisingDP} architecture (PFGM w/ DDPM channel). We use the RK45~\citep{Dormand1980AFO} solver in the Scipy library~\citep{Virtanen2020SciPy1F} for PFGM sampling. We report the FID score using the evaluation protocol in \cite{dhariwal2021diffusion}.
\begin{table}[htbp]
\begin{center}
\caption{FID/NFE on LSUN bedroom $256\times 256$}\label{tab:lsun}
\begin{tabular}{l c c}
\toprule
      &FID $\downarrow$ & NFE $\downarrow$\\
    \midrule
    StyleGAN~\cite{karras2019style} & $\bm{2.65}$ & $\bm{1}$ \\
    DDPM~\cite{Ho2020DenoisingDP} &$6.86$ & $1000$\\
    VE-SDE~\cite{Song2021ScoreBasedGM} & $11.75$  &  $2000$\\
    \midrule
    PFGM w/ NCSN++ channel &  ${17.01}$  &  $134$\\
    PFGM w/ DDPM channel&   ${13.66}$ & ${122}$ \\
     \bottomrule
\end{tabular}
\end{center}
\end{table}

Table~\ref{tab:lsun} shows that PFGM has comparable performance with VE-SDE when using DDPM channel, while achieving around 15$\times$ acceleration. We observe that PFGM achieves a better FID score using the similar configuration in the DDPM model, and converges faster — 150k over the total 2.4M training iterations suggested in \cite{Song2021ScoreBasedGM}. Remarkably, the VE-ODE baseline — the method most comparable to ours — only produces noisy samples on this dataset. It suggests that PFGM is able to scale up to high resolution images when using advanced architectures. 
We also compare with the number reported in \cite{Ho2020DenoisingDP} using similar architecture. Note that DDPM requires 1000 NFE during sampling, and doesn’t possess invertibility compared to flow models.

\subsection{Results on NCSNv2 Architecture}

In this section, we demonstrate the image generation on CIFAR-10 and CelebA $64 \times 64$, using NCSNv2 architecture~\cite{Song2020ImprovedTF}, which is the predecessor of NCSN++ and DDPM++~\cite{Song2021ScoreBasedGM} and has smaller capacity. Since the VE/VP-ODE has poor performance~(FID greater than 90), with the RK45 solver, we also apply the forward Euler method~(\textbf{Euler}) with fixed number of steps. We explicitly name the sampler, with forward Euler method as predictor and Langevin dynamics as corrector, as \textbf{Euler w/ corrector}. For Euler w/ corrector in VE/VP-ODE, we use the probability flow ODE (reverse-time ODE) as the predictor and the Langevin dynamics (MCMC) as the corrector. We borrow all the hyper-parameters from \cite{Song2021ScoreBasedGM} except for the signal-to-noise ratio. We empirically observe the new configurations in Table~\ref{table:s2n} give better results on the NCSNv2 architecture.

To accommodate the extra dimension $z$ on NCSNv2, we concatenate the image with an additional constant channel with value $z$ and thus the first convolution layer takes in four input channels. We also add an additional output channel to the final convolution layer and take the global average pooling of this channel to obtain the direction on $z$.

\begin{table}[htbp]
\begin{center}
\caption{Signal-to-noise ratio of different dataset-method pairs}
\label{table:s2n}
\begin{tabular}{c c c c c c c c}
		\toprule
		\textbf{Dataset-Method} &  CIFAR-10 - VE &CIFAR-10 - VP & CelebA - VE & CelebA - VP\\
		\midrule
        \textbf{signal-to-noise ratio} &  $0.16$ &$0.27$  &$0.12$ & $0.27$ \\
        \bottomrule
\end{tabular}
\end{center}
\end{table}

\label{app:ncsnv2}
\subsubsection{CIFAR-10}

Table~\ref{tab:cifar-ncsnv2} reports the image quality measured by Inception/FID scores and the inference speed measured by NFE on CIFAR-10, using a weaker architecture NCSNv2~\cite{Song2020ImprovedTF}. We show that PFGM with the RK45 solver has competitive FID/Inception scores with the Langevin dynamics, which was the best model on the NCSNv2 architecture before, and requires $10\times$ less NFE. In addition, PFGM performs better than all the other ODE samplers. Our method is more tolerant of sampling error. Among the compared ODEs, our backward ODE~(\Eqref{eq:backode}) is the only one that successfully generates high quality samples while the VE/VP-ODE fail w/o the Langevin dynamics corrector. The backward ODE still beats the baselines w/ corrector. 

\begin{table}[htbp]
    \small
    \centering
    \caption{CIFAR-10 sample quality~(FID, Inception) and number of function evaluation~(NFE). All the methods below the \textit{NCSNv2 backbone} separator use the NCSNv2~\citep{Song2020ImprovedTF} network architecture as the backbone.}
    \begin{tabular}{l c c c}
    \toprule
         & Inception $\uparrow$  &FID $\downarrow$ & NFE $\downarrow$\\
         \midrule
         PixelCNN~\citep{Oord2016ConditionalIG} & $4.60$ & $65.93$ & $1024$\\
        IGEBM~\citep{Du2019ImplicitGA} & $6.02$ & $40.58$ & $60$\\
        WGAN-GP~\citep{Gulrajani2017ImprovedTO} & $7.86 \pm .07$ & $36.4$& $1$\\
        SNGAN~\citep{Miyato2018SpectralNF} & $8.22\pm .05$ & $21.7$ & $1$\\
        NCSN~\citep{Song2019GenerativeMB} & $\bm{8.87 \pm .12}$ & $25.32$ & $1001$\\
        \midrule
        \textit{\textbf{NCSNv2 backbone}}\\
        \midrule
        Langevin dynamics~\citep{Song2020ImprovedTF} & $8.40 \pm .07$ & $\bm{10.87}$ & $1161$\\
        VE-SDE~\citep{Song2021ScoreBasedGM} & $8.23 \pm .02$& $10.94$&$1000$\\
        VP-SDE~\citep{Song2021ScoreBasedGM} & $6.85 \pm .01$& $44.05$&$1000$\\
        \midrule
        VE-ODE~(Euler w/ corrector) & $8.05 \pm .03$ &  $11.33$ &  $1000$\\
        VP-ODE~(Euler w/ corrector) & $7.33\pm .07$ &  $37.74$ & $1000$ \\
        PFGM~(Euler)& $8.00\pm .09$  &  $11.78$ & $200$ \\
        PFGM~(RK45)& $8.30 \pm .05 $ &  $11.22$ & $\bm{118}$ \\
         \bottomrule
    \end{tabular}
    \label{tab:cifar-ncsnv2}
\end{table}

\subsubsection{CelebA}

In Table~\ref{tab:celeba}, we report the quality of images generated by models trained on CelebA $64 \times 64$, as measured by the FID scores, and the sampling speed, as measured
by NFE. We use this dataset as our preliminary experiments hence we only apply NCSNv2~\cite{Song2020ImprovedTF} for different baselines. As shown in Table~\ref{tab:celeba}, PFGM achieves best FID scores than all the baselines on CelebA dataset, while accelerating the inference speed around $20\times$. Remarkably, PFGM outperforms the Langevin dynamics and reverse-time SDE samplers, which are usually considered better than their deterministic counterparts. 

\paragraph{Remark: On the FID scores on CelebA $64 \times 64$} {One interesting observation is that the samples of PFGM (RK45)~(\Figref{pfgm_ncsnv2_celeba}) contain more obvious artifacts than Langevin dynamics~(\Figref{fig:ncsnv2_celeba}), although PFGM has a lower FID score on the same architecture. We hypothesize that the diversity of samples has larger effects on the FID scores than the artifacts. As shown in \Figref{fig:ncsnv2_celeba} and \Figref{pfgm_ncsnv2_celeba}, samples generated by PFGM have more diverse background colors and hair colors than samples of Langevin dynamics. In addition, we evaluate the performance of PFGM on the DDPM++ architecture. We show that the FID score can be further reduced to $3.68$ using the more advanced DDPM++ architecture. By examining the generated samples of PFGM on DDPM++~(\Figref{fig:extend-celeba}), we observe that the samples are diverse and exhibit fewer artifacts than PFGM on NCSNv2. It suggests that by using a more powerful architecture like DDPM++, we can remove the artifacts while retaining the diversity in PFGM.}

\begin{figure}[htbp]
    \centering
    \subfigure[Langevin dynamics~\cite{Song2019GenerativeMB}]{\label{fig:ncsnv2_celeba}\includegraphics[width=0.49\textwidth]{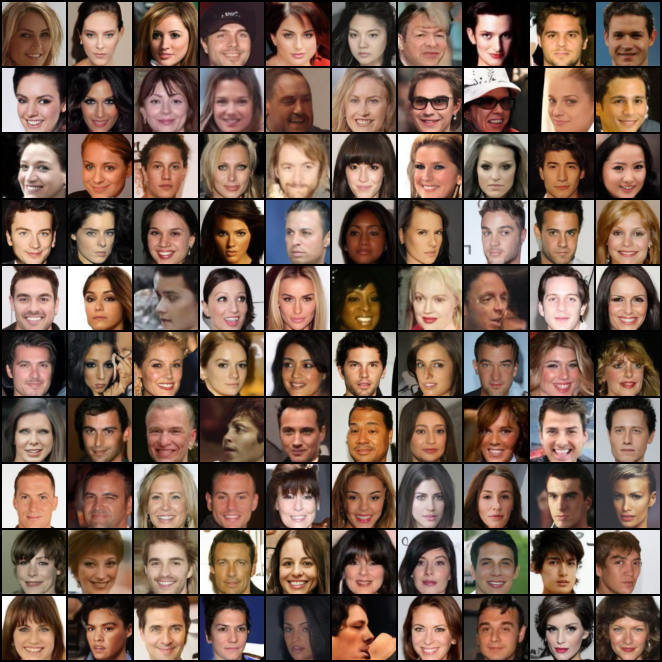}}\hfill
    \subfigure[PFGM~(RK45)]{\label{pfgm_ncsnv2_celeba}\includegraphics[width=0.49\textwidth]{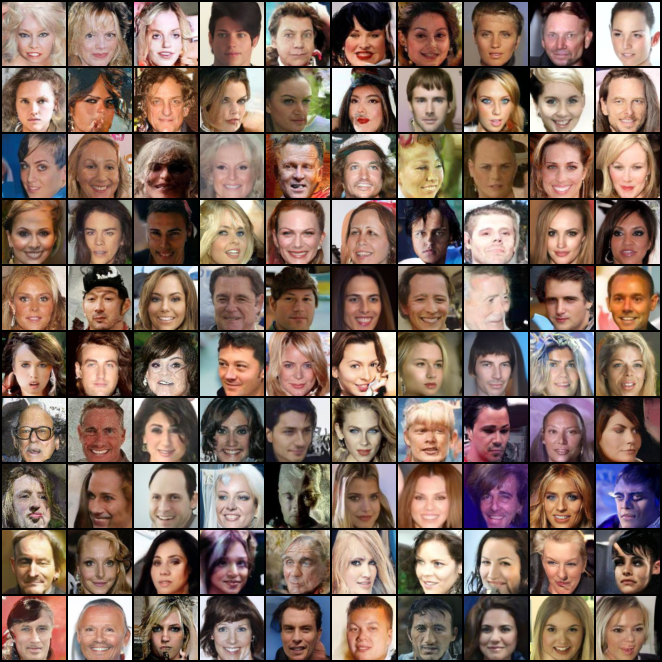}}\hfill
    \caption{Uncurated samples from Langevin dynamics~\cite{Song2019GenerativeMB} and PFGM~(RK45), both using the NCSNv2 architecture.}
    \label{fig:celeba_examine}
\end{figure}

\begin{table}[htbp]
\begin{center}
\caption{FID/NFE on CelebA $64 \times 64$}\label{tab:celeba}
\begin{tabular}{l c c}
\toprule
      &FID $\downarrow$ & NFE $\downarrow$\\
     \midrule
     NCSN~\citep{Song2019GenerativeMB} & $26.89$  & $1001$\\
     \midrule
    \textit{\textbf{NCSNv2 backbone}}\\
    \midrule
    Langevin dynamics~\citep{Song2020ImprovedTF} & $10
    .23$  & $2501$\\
    VE-SDE~\citep{Song2021ScoreBasedGM} & $8.15$ & $1000$\\
    VP-SDE~\citep{Song2021ScoreBasedGM} &$34.52$ & $1000$\\
    \midrule
    VE-ODE~(Euler w/ corrector)  & $8.30$  &  $200$\\
    VP-ODE~(Euler w/ corrector)  &  $41.81$ &  $200$\\
    PFGM~(Euler) &  ${7.85}$  &  $\bm{100}$\\
    PFGM~(RK45)&   ${7.93}$ & ${110}$ \\
    \midrule
    \textit{\textbf{DDPM++ backbone}}\\
    \midrule
     PFGM~(RK45) & $\bm{3.68}$ & $110$\\
     \bottomrule
\end{tabular}
\end{center}
\end{table}

\subsection{{Wall-clock Sampling Time}}

The main bottleneck of sampling time in each ODE step is the function evaluation of the neural network. Hence, for different ODE equations using similar neural network architectures, their inference times per ODE step are approximately the same.

We implement PFGM on the NCSNv2~\cite{Song2020ImprovedTF}, DDPM++~\cite{Song2021ScoreBasedGM}, and DDPM++ deep~\cite{Song2021ScoreBasedGM} architectures, with sight modifications to account for the extra dimension $z$. In Table~\ref{table:wall-clock}, we report the sampling time per ODE step method with the DDPM++ backbone, as well as the total sampling time. We measure the sampling time of generating a batch of 1000 images on CIFAR-10. We compare PFGM, VP/sub-VP ODEs using the RK45 solver. As a reference, we also report the results of VP-SDE using the predictor-corrector sampler~\cite{Song2021ScoreBasedGM}. All the numbers are produced on a single NVIDIA A100 GPU.

\begin{table*}[htb]
\begin{center}
\caption{Wall-clock sampling time~(second)}
\label{table:wall-clock}
\begin{tabular}{c c c c c c c c}
		\toprule
		\textbf{Method} &  PFGM &VP-ODE& sub-VP-ODE& VP-SDE (PC)\\
		\midrule
        \textbf{NFE} & 110 & 134 &146 & 1000\\
		\midrule
        \textbf{Wall-clock time per step} & 0.526 & 0.522 &0.520 & 0.491\\
		\midrule
        \textbf{Total wall-clock time} & 57.81 & 69.97 &75.92&490.65\\
        \bottomrule
\end{tabular}
\end{center}
\end{table*}

As expected, ODEs using similar architectures and the same solver have nearly the same wall-clock time per ODE step. The table also shows that PFGM achieves the smallest total wall-clock sampling time. 

\subsection{Image Interpolations}
\label{app:interpolate}

The invertibility of the ODE in PFGM enables the interpolations between pairs of images. As shown in \Figref{fig:interpolation}, we adopt the spherical interpolations between the latent representations of the images in the first and last column.

\begin{figure}[htbp]
    \centering
    \includegraphics[width=0.8\textwidth]{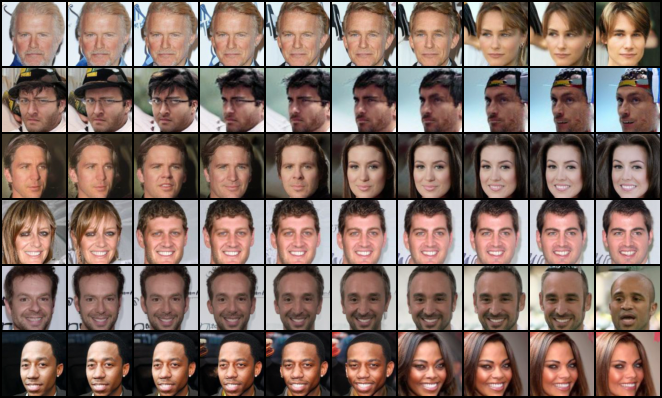}
     \caption{Interpolation on CelebA $64\times 64$ by PFGM}\label{fig:interpolation}
\end{figure}

\subsection{Temperature Scaling}
\label{app:temp}

To demonstrate more utilities of the meaningful latent space of PFGM, we include the experiments of temperature scaling on CelebA $64 \times 64$ dataset. We linearly increase the norm of latent codes from $1000$ to $6000$ to get the samples in \Figref{fig:temperature}.
\begin{figure}[htbp]
    \centering
    \includegraphics[width=0.8\textwidth]{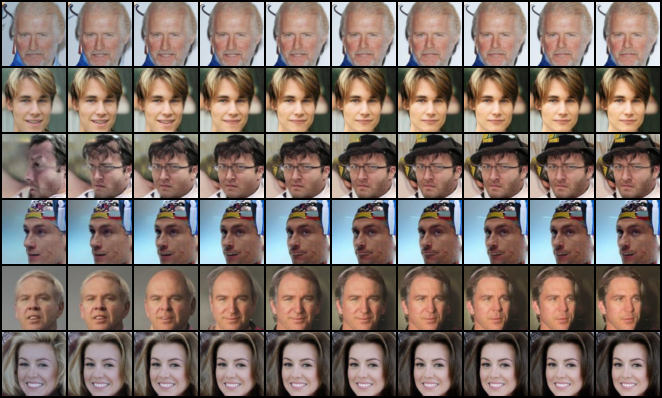}
    \caption{Temperature scaling on CelebA $64 \times 64$ by PFGM}
    \label{fig:temperature}
\end{figure}

\section{Extended Examples}
\label{app:samples}
We provide extended samples from PFGM on CIFAR-10~(\Figref{fig:extend-cifar}), CelebA $64 \times 64$~(\Figref{fig:extend-celeba}) and {LSUN bedroom $256 \times 256$~(\Figref{fig:extend-bedroom-ddpm}) datasets.}

\section{Physical Interpretation of the ODEs in PFGM}\label{app:honey}

{In Section~\ref{section:bg}, in order to move the particles along the electric lines, we set the time derivative of $x$ to the Poisson field $\mat{E}(x)$:}
\begin{equation}\label{eq:ode2}
    [q=1, {\rm forward\  ODE}]\quad \frac{d\mat{x}}{dt} = \mat{E}(\mat{x}), \quad [q=-1, {\rm backward\  ODE}]\quad \frac{d\mat{x}}{dt} = -\mat{E}(\mat{x})
\end{equation}
{{We give the interpretation of the ODEs from a physical perspective. Newton's law implies that the external force is proportional to the acceleration of the particle. In the overdamped limit, e.g., when the particle is moving in honey, the external force is instead proportional to the velocity of the particle, making the equation of motion a first-order ODE.} 
Denoting the viscosity of the fluid as $\gamma$, the dynamics of the particle under the influence of the electric field of the source $\rho(\mat{x})$ is}
$$
    m\frac{d^2\mat{x}}{dt^2}=-\gamma\frac{d\mat{x}}{dt} + q\mat{E}(\mat{x}),
$$
{which has an overdamped limit  $\frac{d\mat{x}}{dt}=q\mat{E}(\mat{x})$ when we set $t\to\gamma t$ and $\gamma\to\infty$. In this case, a particle with mass $m=1$ and charge $q=1$ would follow the electric field with velocity equal to $\mat{E}$, justifying Eq.~(\ref{eq:ode2}).}

\section{Limitations and Future Directions}
\label{app:limit}

In Section~\ref{sec:learning} we discuss the training paradigm of PFGM, including the normalized Poisson field and the discretized forward ODE. There are several potential improvements. First, the normalized field on mini-batch is biased. In this paper, we directly alleviate the bias by using a larger training batch. However, it does not solve the problem fundamentally. Some potential directions are incorporating more physical tools: we can exploit renormalization to make the Poisson field well-behaved in near fields. Another possibility is to replace a point charge with a quantum particle, whose position uncertainty fills the empty space among nearest neighbor data samples and makes the data manifold smoother.

\section{Potential Social Impact}
\label{app:impact}

Generative models is a rapidly growing field of study with far-reaching implications for science and society. 
Our work proposes a new generative model that allows for high-quality samples, quick inference and adaptivity. Many downstream applications benefit from our PFGM models' powerful expressive capabilities, particularly those that need fast inference speed and good sample quality at the same time. The usage of these models might have both positive and negative outcomes depending on the downstream use case. 
For example, PFGM can be incorporated in producing good image/audio samples by the fast backward ODE. This, on the other hand, promotes \textit{deepfake} technology and leads to social scams. Generative models are also brittle and susceptible to backdoor adversarial attacks on publicly available training data, causing unanticipated failure.
Addressing the above concerns requires further research in providing robustness guarantees for generative models as well as close collaborations with researchers in socio-technical disciplines.

\begin{figure*}
    \centering
    \includegraphics[width=0.9\textwidth]{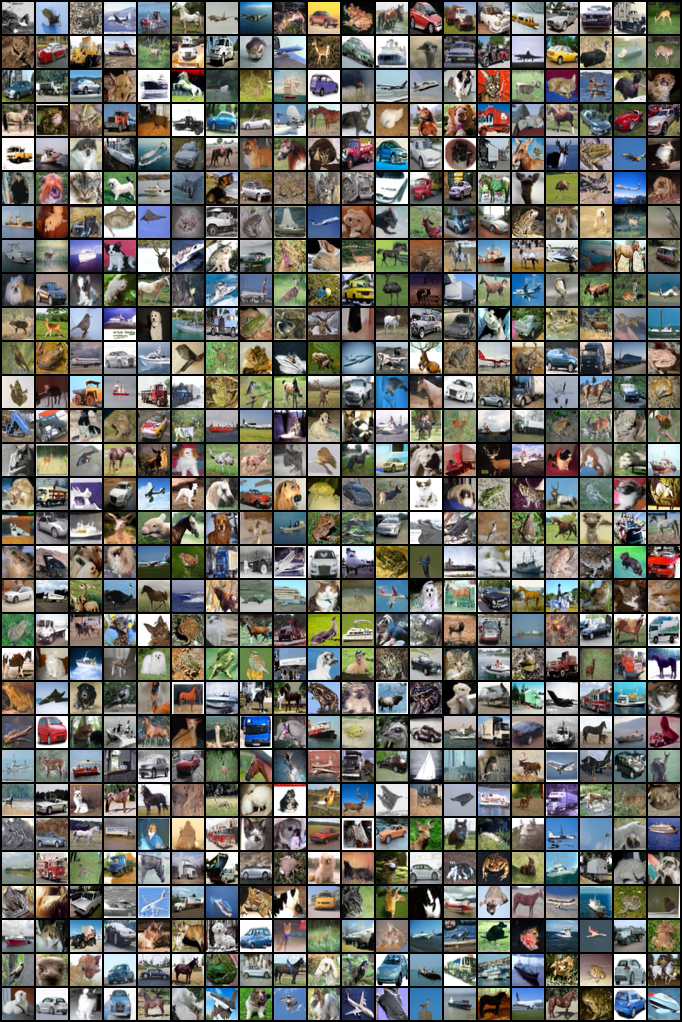}
    \caption{CIFAR-10 samples from PFGM (RK45)}
    \label{fig:extend-cifar}
\end{figure*}

\begin{figure*}
    \centering
    \includegraphics[width=0.9\textwidth]{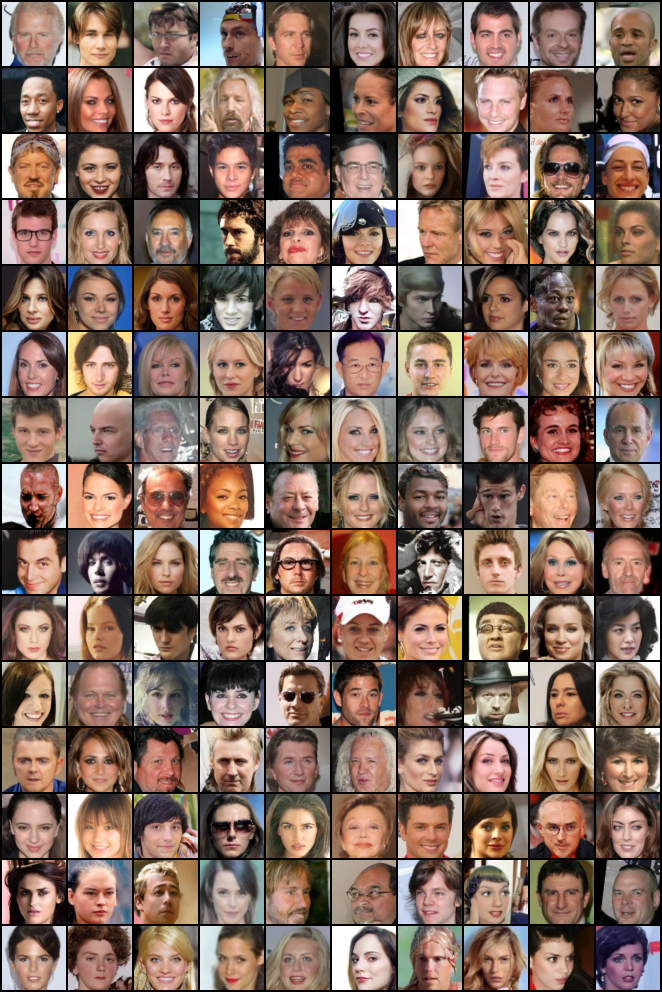}
    \caption{CelebA $64 \times 64$ samples from PFGM (RK45, NCSNv2 architecture)}
    \label{fig:extend-celeba}
\end{figure*}


\begin{figure*}
    \centering
    \includegraphics[width=0.9\textwidth]{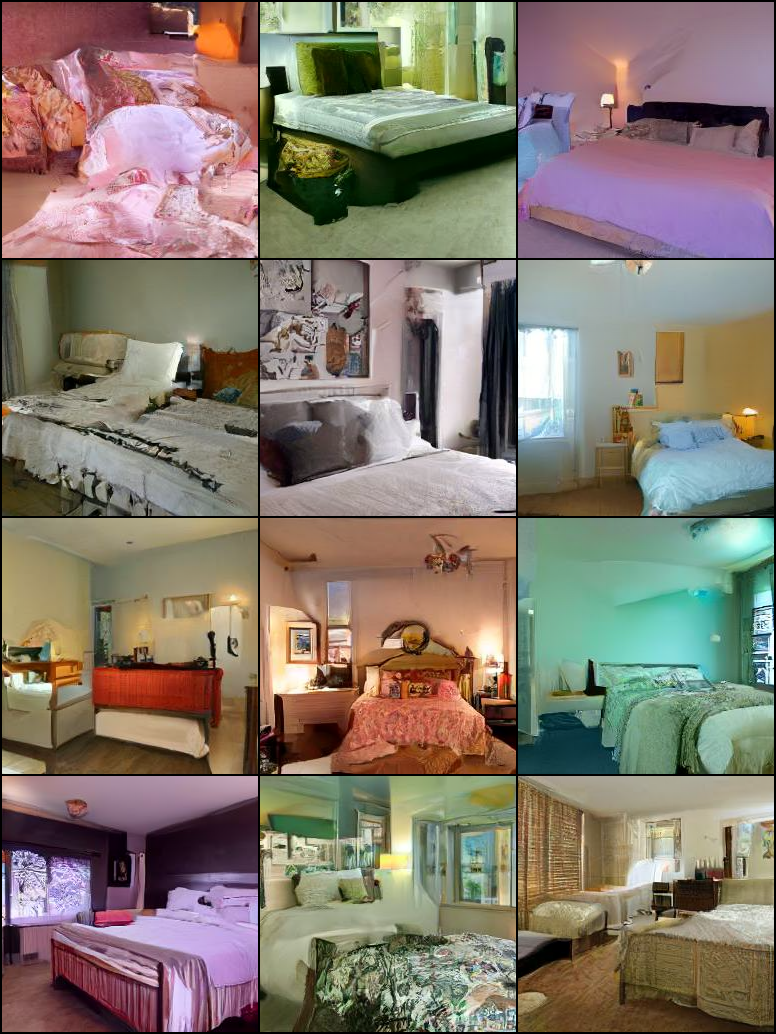}
    \caption{LSUN bedroom $256 \times 256$ samples from PFGM (RK45) using DDPM channel configuration.}
    \label{fig:extend-bedroom-ddpm}
\end{figure*}
\clearpage


\end{document}